\documentclass{article}
\usepackage[margin=1in]{geometry}
\usepackage{amsthm}

\usepackage{graphicx} %
\usepackage{amsmath}
\usepackage{amssymb}
\usepackage{mathtools}
\usepackage{soul}
\usepackage{stmaryrd}
\usepackage{enumitem}
\usepackage{comment}
\usepackage{rotating}
\usepackage{booktabs}       %
\usepackage{amsfonts}       %
\usepackage{nicefrac}       %
\usepackage{microtype}      %
\usepackage{multirow}
\usepackage{xcolor}
\usepackage{hyperref}
\usepackage{subcaption}
\usepackage{tabularx}
\usepackage[font=small,skip=4pt]{caption}
\allowdisplaybreaks

\usepackage{algorithm}
\usepackage{algcompatible}
\makeatletter
\AddToHook{env/algorithmic/begin}{\def\@currentcounter{ALG@line}}
\makeatother
\usepackage{cleveref}
\crefalias{ALG@line}{line}

\newtheorem{proposition}{Proposition}
\newtheorem{theorem}{Theorem}
\newtheorem{lemma}{Lemma}
\newtheorem{corollary}{Corollary}
\newtheorem{corollary*}{Corollary}
\newtheorem{definition}{Definition}

\newtheorem{remark}{Remark}
\crefname{corollary}{corollary}{Corollary}
\crefname{lemma}{lemma}{Lemma}
\crefname{definition}{definition}{Definition}
\crefname{example}{example}{Example}

\newcommand{\relu}{\text{ReLU}}

\renewcommand{\L}{\mathcal{L}}

\newcommand{\by}{\mathbf{y}}
\newcommand{\bx}{\mathbf{x}}

\newcommand{\bu}{\mathbf{u}}

\newcommand{\bw}{\mathbf{w}}

\newcommand{\CP}{\mathcal{P}}
\newcommand{\CR}{\mathcal{R}}
\newcommand{\din}{d_{\textnormal{in}}}
\newcommand{\dout}{d_{\textnormal{out}}}

\title{Multilevel Training for Kolmogorov Arnold Networks}
\author{Ben S. Southworth\thanks{Theoretical Division, Los Alamos National Laboratory, USA.}
   \and Jonas A. Actor\thanks{Center for Computing Research, Sandia National Laboratories, USA.}
   \and Graham Harper\footnotemark[3]
   \and Eric C. Cyr\footnotemark[3]
}

\ifpdf
\hypersetup{ pdftitle={Multilevel Training for KANs} }
\fi

\begin{document}
\maketitle
\allowdisplaybreaks

\begin{abstract}
Algorithmic speedup of training common neural architectures is made difficult by the lack of structure guaranteed by the function compositions inherent to such networks. In contrast to multilayer perceptrons (MLPs), Kolmogorov-Arnold networks (KANs) provide more structure by expanding learned activations in a specified basis. This paper exploits this structure to develop practical algorithms and theoretical insights, yielding training speedup via multilevel training for KANs. To do so, we first establish an equivalence between KANs with spline basis functions and multichannel MLPs with power ReLU activations through a linear change of basis. We then analyze how this change of basis affects the geometry of gradient-based optimization with respect to spline knots. The KANs change-of-basis motivates a multilevel training approach, where we train a sequence of KANs naturally defined through a uniform refinement of spline knots with analytic geometric interpolation operators between models. The interpolation scheme enables a ``properly nested hierarchy'' of architectures, ensuring that interpolation to a fine model preserves the progress made on coarse models, while the compact support of spline basis functions ensures complementary optimization on subsequent levels. Numerical experiments demonstrate that our multilevel training approach can achieve orders of magnitude improvement in accuracy over conventional methods to train comparable KANs or MLPs, particularly for physics informed neural networks. Finally, this work demonstrates how principled design of neural networks can lead to exploitable structure, and in this case, multilevel algorithms that can dramatically improve training performance. 
\end{abstract}

\section{Introduction}\label{sec:intro}

Multilayer perceptrons (MLPs) \cite{mcculloch1943logical, rosenblatt1958perceptron, pinkus1999approximation} are a classical deep learning architecture that exploit the composition of affine maps with a nonlinear scalar activation function. MLP architectures appear as blocks in many state-of-the-art applications, including (but not limited to) variational autoencoders \cite{kingma2013auto} and transformers \cite{vaswani2017attention}. Training MLPs, especially modern multichannel and multihead variants of state-of-the-art architectures, is a nontrivial numerical task. Typically these methods rely on iterative variants of stochastic gradient descent resulting in relatively slow convergence. In contrast, in classical numerical methods, multilevel and multigrid methods for solving numerical partial differential equations (PDEs), are some of the most powerful solvers of linear and nonlinear equations, capable of solving a sparse $n\times n$  linear system in $\mathcal{O}(n)$ operations.

Motivated by this success, there have been a number of attempts to extend multigrid ideas to machine learning, e.g. \cite{Chang2018,Gunther2020,cyr2025torchbraid,ke2017multigrid,He2019,Eliasof2023,Kopanicakova2022,GaedkeMerzhauser2021,Feischl2024t1}. The connection between deep neural networks and discretized dynamical systems \cite{haber2017stable, ruthotto2020deep, chen2018neural} provides theoretical motivation for applying multilevel methods from numerical PDEs to neural network training, as ResNet depth can be interpreted as a time discretization. Building on this interpretation, nested iteration via multilevel-in-layer is considered in \cite{Chang2018} observing modest speedups, and multigrid-in-time concepts are applied to develop a layer-parallel training strategy in \cite{Gunther2020,cyr2025torchbraid}, but the obtained speedups are due to parallelization and are not algorithmic. Several works have also designed architectures that explicitly incorporate multigrid structure, e.g. \cite{ke2017multigrid,He2019,Eliasof2023}. While these approaches demonstrate that multigrid principles can inform architecture design, they focus on model structure rather than training algorithms and do not report training speedups. Rigorous multilevel training in the context of nonlinear multigrid and cycling is considered in \cite{Kopanicakova2022,GaedkeMerzhauser2021}, demonstrating modest convergence improvements. Most recently, Feischl et al. \cite{Feischl2024t1} develops a theoretical framework for refinement of feed-forward neural networks that is incorporated into the optimization procedure.

Nevertheless, to the best of our knowledge no multilevel machine learning works have provided \emph{algorithmic} speedups like seen in other fields such as numerical PDEs. Broadly, this is due to the lack of multilevel machine learning hierarchies with good approximation properties between levels, and grid-specific optimization or ``relaxation'' routines that complement the chosen coarsening and interpolation. Implicit in this context is the substantial difficulty in defining coarse representations of a machine learning model. Because in machine learning the coarse and fine models typically operate on the same dimensional space, there is no clear extension of approximation properties from multigrid literature to motivate the choice of coarse model and transfer operators. On a high level, coarse models must be (i) cheaper to (approximately) solve than the fine model, (ii) not conflict with the fine model objective, and moreover (iii) provide correction or descent direction that is difficult to capture/identify with the fine model, i.e. be complementary to ``relaxation,'' or grid-specific optimization methods. %

An alternative architecture to MLPs, Kolmogorov-Arnold Networks (KANs) \cite{liu2024kan}, has gained popularity in recent literature, having been used for a range of tasks including computer vision \cite{cang2024can, cacciatore2024preliminary, cheon2024demonstrating}, time series analysis \cite{vaca2024kolmogorov, dong2024kolmogorov}, scientific machine learning \cite{abueidda2024deepokan, toscano2024pinns, wu2024kolmogorov, koenig2024kan, jacob2024spikans, patra2024physics, howard2024finite, rigas2024adaptive}, graph analysis \cite{kiamari2024gkan, bresson2024kagnns}, and beyond; see \cite{somvanshi2024survey, hou2024comprehensive} and references therein for a list of versions, extensions, and applications of KANs, and \cite{noorizadegan2025practitioner} for a comprehensive review. The structural similarity between MLPs and KANs has been noted multiple times in the literature, e.g.
\cite{pinkus1999approximation, igelnik2003kolmogorov, leni2013kolmogorov, he2024mlp, qiu2025powermlp}.  
Due to their similarities, KANs tend to be comparable to MLPs for learning tasks, with the same asymptotic complexity and convergence rates \cite{gao2024convergence} and performance in  variety of head-to-head comparisons \cite{zeng2024kan, shukla2024comprehensive, yu2024kan}. Compared to MLPs, KANs in particular are known for (i) being more interpretable, as the model output consists of analytical functional composition, and (ii) being able to better capture low-regularity solutions and mappings than traditional MLPs. However, theoretical insights and practical algorithms for KANs that exhibit these properties while maintaining (or outperforming) MLPs are generally inconsistent or lacking, and it is this gap that this paper aims to address, with the particular aim of advancing training strategies to exploit multilevel optimization algorithms.

\subsection{Contributions}

In this paper, we advance theoretical and practical insights pertaining to KANs, improving upon our conceptual understanding of how architectures and optimizers pair together to achieve better training and model performance. We translate these insights into a demonstration that shows how multilevel methods can dramatically improve training of a properly designed neural network. To do so, we:
\begin{enumerate}
\item exploit the relationship between KANs and multichannel MLPs to introduce a change-of-basis map between the two architectures;
\item analyze how this change of basis alters the dynamics of gradient descent; and
\item introduce the concept of \emph{properly nested hierarchy} for multilevel optimization, show that KANs with appropriate interpolation operators satisfy this ansatz, and design and demonstrate a corresponding multilevel training approach inspired by multigrid methods.
\end{enumerate}
The outline of the paper and the technical steps and reasons that necessitate the accompanying analysis are described next.

\subsection{The rest of the paper}
We begin in \Cref{sec:kans-mlp:cob} by exploiting the relationship between B-Splines and ReLU activations to reformulate KANs in the language of multichannel MLPs. We show that KAN architectures with spline activations are equivalent to a certain form of multichannel MLPs under an appropriate linear change of basis, where the biases are fixed to match the spline knots. In \Cref{sec:feature:eval}, we show that the linear change of basis between KANs and multichannel MLPs exactly matches a finite-difference discretization of the $r$th derivative operator on spline knots. This change of basis immediately yields a direct, non-recursive implementation of spline basis KANs that is faster than the typical Cox-de Boor recursive form (see \Cref{sec:kans-mlp:cost}). 

Despite the equivalence (as forward operators) under the change of basis, gradient-based training of KANs and multichannel MLPs yield fundamentally distinct weight evolution. This is discussed abstractly in terms of the geometry of the primal and dual space in \Cref{sec:feature}, where a function of the change of basis matrix between KANs and multichannel MLPs acts as a preconditioner of descent methods depending on your choice of spaces. Combining these results with the analysis from \Cref{sec:feature:eval} shows substantial benefit to the training dynamics due to the eigenstructure of differential operators: the multichannel MLP formulation will strongly prioritize training smooth functions across spline knots. This makes an expressive model of complex functions easier for the KAN to realize, as compared to a multichannel MLP.

In \Cref{sec:multilevel}, we exploit the mathematical formulations from \Cref{sec:kans-mlp} and \Cref{sec:feature} to develop an efficient multilevel training framework for spline-based KANs, built on the structure that comes with the spline parameterization, and structure that MLPs lack. 
To do this, we introduce a new concept of \emph{properly nested hierarchies} for multilevel optimization, which ensures that interpolation to a fine model does not \emph{undo} progress made on the coarse model.  Our transfer operators are built geometrically, naturally yielding a properly nested hierarchy, and accelerating the grid transfer methods posed in \cite{liu2024kan} to be fast enough to be constructed/applied during training. The theory developed in \Cref{sec:feature} then indicates that gradient-based optimization applied to the properly nested coarse and fine models are \emph{complementary}, a fundamental requirement for any successful multigrid method. Altogether, we believe that a properly nested hierarchy with level-complementary gradient-based optimization routines provide the necessary components for a successful multigrid method. 

In \Cref{sec:results}, we then numerically demonstrate the multilevel training to provide significant improvements in both accuracy and efficiency applied to functional regression and physics informed neural networks (PINNs).
We similarly show for the same problems that multilevel training applied to the equivalent multichannel MLP basis, which lacks the level-complementary gradient-based optimization, yields effectively zero improvement over just the coarse model. This is because training of the fine-grid multichannel MLP prioritizes modes already captured by the coarse model. %

\section{KANs and multichannel MLPs}\label{sec:kans-mlp}
We begin our analysis with an introduction of KANs in the spline basis, and then demonstrate how a linear transformation relates each layer in the spline basis to an equivalent multichannel MLP. From a historical and approximation theory standpoint \cite{lorentz1962metric, sprecher1993universal, pinkus1999approximation, sprecher2017algebra}, both KANs and MLPs stem from arguments that expand upon the Kolmogorov Superposition Theorem (KST) \cite{kolmogorov1957representation}, also called the Kolmogorov-Arnold Superposition Theorem. This theorem states there exist functions $\phi_{pq} \in C([0,1])$, for $p=1,\dots,n$ and $q=1,\dots,2n+1$, such that for any function $f \in C([0,1]^n)$, there are functions $\varphi_{q} \in C(\mathbb{R})$ such that 
$f(x_1,\dots,x_n) = \sum_{q=1}^{2n+1} \varphi_{q}\left( \sum_{p=1}^n \phi_{pq}(x_p) \right).$
Classically $\varphi_{q}$ and $\phi_{pq}$ are not differentiable at a dense set of points and only H\"older or at best Lipschitz continuous \cite{braun2009constructive, sprecher1993universal,actor2018computation}.
Inspired by the KST, KANs replace the functions $\varphi_{q}$ and $\phi_{pq}$ 
 with learned functions expressed in a fixed basis, often a spline basis \cite{liu2024kan}.  While the original KST invokes a single layer of function composition creating a shallow network, KANs use more layers of function composition to increase the depth of the resulting representation. For a layer $\ell$ with $P$ inputs and $Q$ outputs define the output of a KAN layer as
 $x^{(\ell+1)}_q = \sum_{p=1}^P \phi^{(\ell)}_{pq}( x^{(\ell)}_p ) \mbox{ for } q=1,\dots,Q,$
with $\phi^{(\ell)}_{pq}$ trained from data.
In practice, each $\phi_{pq}^{(\ell)}$ must be represented in a computationally tractable form. Following 
\cite{liu2024kan}, we expand each $\phi_{pq}^{(\ell)}$ in a basis of B-splines of order $r$, 
yielding a finite-dimensional parameterization where the spline coefficients become the trainable 
weights of the network. This choice of basis is motivated by the approximation properties of 
splines and their compact support, which we will exploit throughout this work. Specifically, 
given a set of knots, each $\phi_{pq}^{(\ell)}$ is represented as
$\phi_{pq}^{(\ell)}(x) = \sum_{i=1-r}^{n-1} \widetilde{W}^{(\ell)}_{qpi} \, b_i^{[r]}(x),$
where $\widetilde{W}^{(\ell)}_{qpi}$ are learnable weights and $\{b_i^{[r-1]}\}_{i=1-r}^{n-1}$ forms a B-spline basis.

\subsection{Change of basis}\label{sec:kans-mlp:cob}

We formalize the structure of the spline basis to establish an equivalence to an alternative representation using ReLU activations via a linear change of basis. Let $T = \{t_i\}_{i=1-r}^{n+r-1}$ be a strictly ordered\footnote{We assume a strictly ordered, i.e. nondegenerate, set of spline knots, which simplifies the required assumptions and definitions. While the authors know of no theoretical reasons why degenerate splines should not be considered under this framework, there is substantial technical rigor necessary to treat these properly, so we omit them from this discourse.} set of $n+2r-1$ spline knots, where $t_i < t_{i+1}$, with $t_0 = a$ and $t_n = b$. Let $\mathbb{P}^k$ denote the space of polynomials of degree $k$. The \emph{spline space of order $r$ with knots $T$}, denoted by $S_r(T)$, is the set of functions
\begin{equation}
S_r(T) = \{ f \in C^{r-2}([a,b]) \,\,:\,\, f\lvert_{[t_i, t_{i+1}]} \in \mathbb{P}^{r-1} ,\, t_i \in T\}.
\end{equation}
The spline space of order $r$ consists of piecewise polynomials of degree $r-1$ (subject to additional smoothness constraints). For $r=1$, define $b_i^{[1]}(x) = \begin{cases} 1 & x \in [t_i, t_{i+1}] \\
    0 & \text{else}\end{cases}$
and for $r >1$, we use the Cox-de Boor \cite{deboor1978practical} recursion formula:
\begin{equation}\label{eq:spline-basis}
    b^{[r]}_i(x) = \frac{x - t_i}{t_{i+r-1} - t_i} b^{[r-1]}_i(x) + \frac{t_{i+r} - x}{t_{i+r} - t_{i+1}} b^{[r-1]}_{i+1}(x).
\end{equation}
Each function $b^{[r]}_i$ is supported on the interval $[t_i, t_{i+r}]$.
We additionally define a $\text{ReLU}^{r-1}$ basis, spanned by functions 
\begin{equation}\label{eq:relu-basis}
    \psi^{[r]}_i(x) = \text{ReLU}(x - t_i)^{r-1}.
\end{equation}
We can now restate a classical spline approximation result, cf. \cite{chui1988multivariate}.
\begin{lemma}
Both  $B_S^{[r]} = \{ b^{[r]}_i \}_{i=1-r}^{n-1}$ and $B_R^{[r]} = \{ \psi_i^{[r]} \}_{i=1-r}^{n-1}$
are bases for $S_r(T)$.
\end{lemma}
Thus there exists a linear change of basis between them; denote this change-of-basis matrix as $A^{[r]} \in \mathbb{R}^{(n+r-1) \times (n+r-1)}$, so that
\begin{equation}\label{eq:cob}
    B^{[r]}_S = A^{[r]}B^{[r]}_R,
\end{equation} 
with $B^{[r]}_S$ and $B^{[r]}_R$ understood to be square matrices with columns given by basis vectors. The matrix $A^{[r]}$ is structured and banded, with a particularly elegant expression for the case of uniform knot spacing:
\begin{lemma}\label{lemma:basis}
    For $r > 1$, let $A^{[r-1]}$ denote splines of order $r-1$ constructed on knots for splines of order $r$. Then the matrix $A^{[r]}$ is defined entry-wise as:
    \begin{equation} \label{eq:A_recurrence}
         A^{[r]}_{ij} = \frac{1}{t_{i+r-1} - t_i} A^{[r-1]}_{ij} - \frac{1}{t_{i+r} - t_{i+1}} A^{[r-1]}_{i+1,j},
    \end{equation}
    where we define $A^{[r-1]}_{n+r,:} = 0$ to account for the special case of the final row $i=n+r-1$. 
\end{lemma}
\begin{proof}
    See \Cref{app:cob_proof}.
\end{proof}
\begin{corollary}[Uniform knots]\label{cor:uni}
In the case of uniform knots, with spacing $h = t_i - t_{i-1}$, $A^{[r]}$ is upper triangular Toeplitz, with entries
    \begin{equation*}
        A^{[r]}_{ij} = \begin{cases}
            \frac{(-1)^{j-i} }{(j-i)!\,(r-j+i)!} \, \frac{r}{h^{r-1}}  & i \le j \le i+r \\
            0 & \text{else}
        \end{cases}.
    \end{equation*}
\end{corollary}
\begin{proof}
    See \Cref{app:cob_proof}.
\end{proof}

We now show that the change of basis $A^{[r]}$ \eqref{eq:cob} yields an equivalence between KANs and a certain form of multichannel MLP. Each layer of a KAN architecture in the spline basis $B_S^{[r]}$ has spline coefficients stored as parameters/weights in a 3-tensor $\widetilde{W}^{(\ell)}$ such that
\begin{equation} \label{eq:kan-spline-basis}
\bx_q^{(\ell+1)} = \sum_{p=1}^P \sum_{i=1-r}^{n-1} \widetilde{W}^{(\ell)}_{qpi} \, b_i^{[r]}(\bx_p^{(\ell)}).
\end{equation}
Define $W^{(\ell)} = \widetilde{W}^{(\ell)} \times_3 (A^{[r]})^T$, where $\times_3$ denotes the $3$-mode tensor product\footnote{See \cite{kolda2009tensor} for more details on this operation.}; at the index level, $W^{(\ell)}_{qpj} \coloneqq  \sum_{i=1-r}^{n-1} \widetilde{W}^{(\ell)}_{qpi} A_{ij}^{[r]}$. Substituting the change of basis from \eqref{eq:cob} and definition of $W^{(\ell)}$ into \eqref{eq:kan-spline-basis} yields
\begin{subequations} \label{eq:kan_relu}
\begin{align}   
   \bx_q^{(\ell+1)} & = \sum_{p=1}^P \sum_{i=1-r}^{n-1} \widetilde{W}^{(\ell)}_{qpi} \left( \sum_{j=1-r}^{n-1} A_{ij}^{[r]} \,\relu(\bx^{(\ell)}_p - t_j)^{r-1} \right) \\
   & = \sum_{p=1}^P \sum_{j=1-r}^{n-1} W^{(\ell)}_{qpj} \, \relu(\bx^{(\ell)}_p - t_j)^{r-1}.
\end{align}
\end{subequations}
With a notational shift outlined in \Cref{app:notational_shift}, we observe that \eqref{eq:kan_relu} is an equivalent dense mulitchannel MLP layer with ReLU$^{r-1}$ activations. Altogether, we arrive at the following result.
\begin{lemma}[Equivalence of KANs and multichannel MLPs]
    A single layer of a KAN in the form of \eqref{eq:kan-spline-basis}, with weight three-tensor $\widetilde{W}^{(\ell)}$, 
    is equivalent to a single layer of a multichannel MLP in the form of \eqref{eq:kan_relu} with weight three-tensor $W^{(\ell)} = \widetilde{W}^{(\ell)} \times_3 (A^{[r]})^T$.
\end{lemma}

Now consider a full network consisting of multiple layers. For ease of notation, and with Section \ref{sec:feature} in mind, let us vectorize weights across all layers, and denote $\bu$ as the full vector of weights in the spline basis and $\bw$ as the vector of weights in the ReLU basis; both $\bu,\,\bw \in \mathbb{R}^{N_w}$. Following the above discussion, with appropriate concatenation and reordering/flattening of $A^{[r]}$ we can define the change of basis matrix $A$ such that
\[
    A^T\bu = \bw.
\]
Note, here the change of basis maps from weights in the spline basis to weights in the ReLU basis, \emph{opposite of the direction of the functional change of basis} \eqref{eq:cob}. 
Assume the vectorization operation on the weights $\bu$ and $\bw$ is ordered with respect to layer, output neuron, input neuron, and then spline knot. With this ordering the resulting $A$ is block diagonal over the first three dimensions (layer, output neuron, and input neuron) with diagonal blocks given by $(A^{[r]})$ of the appropriate dimension over spline knots. 

\subsection{Eigenstructure of $A^{[r]}$}\label{sec:feature:eval}

We now prove the following lemma regarding the structure of $A^{[r]}$ as a discrete approximation of certain differential operators.
\begin{lemma}\label{lem:deriv}
    Up to constant scaling by $\pm(r-1)!/h$ (with sign depending on $r$ being odd or even), $A^{[r]}$ is  a forward finite difference approximation of the $r$th derivative on a 1d uniform grid of mesh spacing $h$, with strongly enforced zero Dirichlet boundary conditions. Up to boundary nodes and constant scaling, $(A^{[r]})^TA^{[r]}$ is also Toeplitz and a finite difference approximation of the $(2r)$th derivative on a 1d uniform grid of mesh spacing $h$.
\end{lemma}
\begin{proof}
Let us consider the stencil $\mathbf{v}^{(r)}$ for spline power $r$ starting from the diagonal and extending to upper triangular indices. Let the diagonal index correspond to zero and $\mathbf{v}^{(r)}_i$ denote the $i$th stencil index. Then 
\begin{align}
    \mathbf{v}^{(r)}_i & \coloneqq \frac{(-1)^{i} }{i!\,(r-i)!} \, \frac{r}{h^{r-1}}
    = \frac{1}{(r-1)!h^{r-1}} \Big[ (-1)^{i} {r \choose i} \Big], \quad i\leq r .
\end{align}
Now we recall the general forward finite difference for $f^{(r)}(x)$ centered at $x$ is given by \cite{sloane1995encyclopedia}
\begin{align}
    f^{(r)}(x) & \approx \frac{1}{h^r} \sum_{i=0}^r (-1)^{r-i} {r \choose i} f(x+i h).
\end{align}
This completes the proof for $A^{[r]}$, where zero Dirichlet boundary conditions correspond to truncated Toeplitz structure in the final matrix rows. For Toeplitz properties of $(A^{[r]})^TA^{[r]}$ and corresponding stencil values see \cite[Lemma 23]{southworth2019necessary}. Consistent with above discussion, the stencil of the inner Toeplitz operator up to constant global scaling takes the form of the general \emph{central} finite difference approximation to the $(2r)$th derivative,
\begin{align}\label{eq:cent-fd}
    f^{(2r)}(x) & \approx \frac{1}{h^{2r}} \sum_{i=0}^{2r} (-1)^{i} {2r \choose i} f(x+(r-i) h).
\end{align}
\end{proof}

\Cref{lem:deriv} proves a discrete differential relationship between the ReLU and spline bases in a KANs architecture. For example, up to boundary nodes, for $r=1$ we have that $(A^{[r]})^TA^{[r]}$ corresponds to the $[1,-2,1]$ stencil for isotropic diffusion in 1d. This immediately gives us strong intuition on the effects of $(A^{[r]})^TA^{[r]}$ as an operator. Let us consider eigenvalues of the $(2r)$th derivative operator $D^{2r} \coloneqq \frac{d^{2r}}{dx}$ and consider the simplified setting of domain $x\in[0,1]$ with periodic boundary conditions $f^{(k)}(0) = f^{(k)}(1)$ for $k\in\{0,...,2r-1\}$. Consider the Fourier basis for eigenvectors
\begin{equation}
    f_i(x) \coloneqq e^{-2\pi i\mathrm{i}x} \quad i\in\mathbb{Z}.
\end{equation}
Differentiating $2r$ times yields
\begin{equation}
    f_i^{(2r)}(x) = D^{2r} e^{-2\pi i\mathrm{i}x}
        = (-2\pi i\mathrm{i})^{2r}e^{-2\pi i\mathrm{i}x} 
        = (-1)^r(2\pi i)^{2r}e^{-2\pi i\mathrm{i}x},
\end{equation}
for $i\in\mathbb{Z}$.
Thus we have eigenvalues of $D^{2r}$ given by $\lambda \in\{(-1)^r(2\pi i)^{2r}\}$ for $i\in \mathbb{Z}$, with magnitude of eigenvalue $\lambda$ directly related to Fourier frequency of corresponding eigenvector $f_i(x) = e^{-2\pi i\mathrm{i}x}$.

Moving to the discrete setting and spline change of basis, $(A^{[r]})^TA^{[r]}$ does not impose periodic boundary conditions and, by nature of the normal-equation form, eigenvalues will be strictly positive rather than alternating sign with $r$ as in the continuous analysis. This is because $(A^{[r]})^TA^{[r]}$ is approximating alternating signs of the $(2r)$th derivative compared with the FD stencil in \eqref{eq:cent-fd}. However, we expect the broad properties to still hold. First, the eigenvalues will span a large range, and (up to constant scaling) roughly take the form $\ell^{2r}$ for eigenvalue index $\ell\in[1,n]$ for $n$ spline knots. Second, the eigenmodes will be correspondingly ordered with respect to smoothness, with the smallest eigenvalues corresponding to the smoothest eigenmodes and the largest eigenvalues corresponding to the most oscillatory eigenmodes on the spline grid, although this eigenbasis will not be a strict Fourier basis due to non-periodic boundary conditions. Each of these properties is demonstrated for $r=1$ and $r=3$ in \Cref{fig:cob-evec}. \Cref{fig:eval-ratio} then shows the difference in scaling of smooth vs oscillatory modes on a spline grid by $(A^{[r]})^TA^{[r]}$ by considering the ratio of smallest to largest eigenvalues. Continuing with the observation that eigenvalues approximately follow $\ell^{2r}$, we expect the ratio to scale like $n^{2r}$, which is confirmed in \Cref{fig:eval-ratio}. Thus for even a small number of $n=10$ spline knots, the most oscillatory modes are scaled more than $100\times$ stronger than the smoothest modes for $r=1$ and $10^6\times$ stronger for $r=3$. To demonstrate the geometry in a simpler manner, we also plot the number of sign changes in the discrete eigenvectors in \Cref{fig:sign-change}. Here we see the expected simple structure, independent of $r$ -- the eigenvector associated with the smallest eigenvalue of $(A^{[r]})^TA^{[r]}$ has no sign changes and the eigenvector associated with the largest eigenvalue has $n-1$. Naturally $(A^{[r]}(A^{[r]})^T)^{-1}$ follows the exact opposite ordering.
\begin{figure}[!bht]
    \begin{subfigure}[t]{\textwidth}
        \centering
        \includegraphics[width=0.9\linewidth]{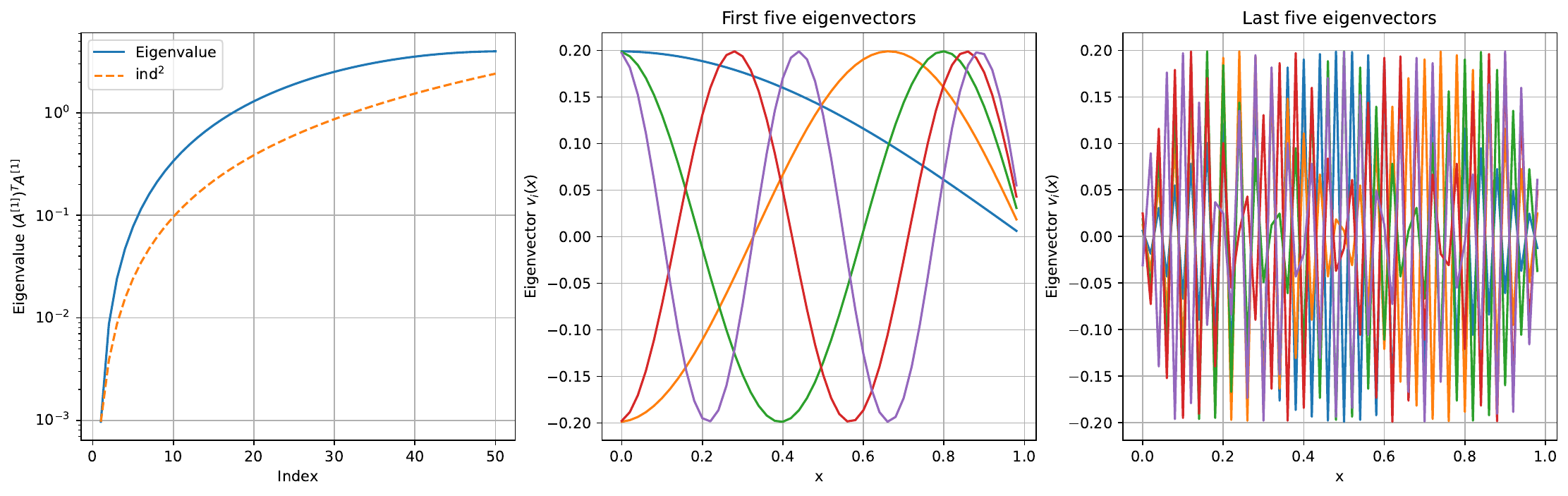}
        \caption{$r=1$}
    \end{subfigure}
    \\
    \begin{subfigure}[t]{\textwidth}
        \centering
        \includegraphics[width=0.9\linewidth]{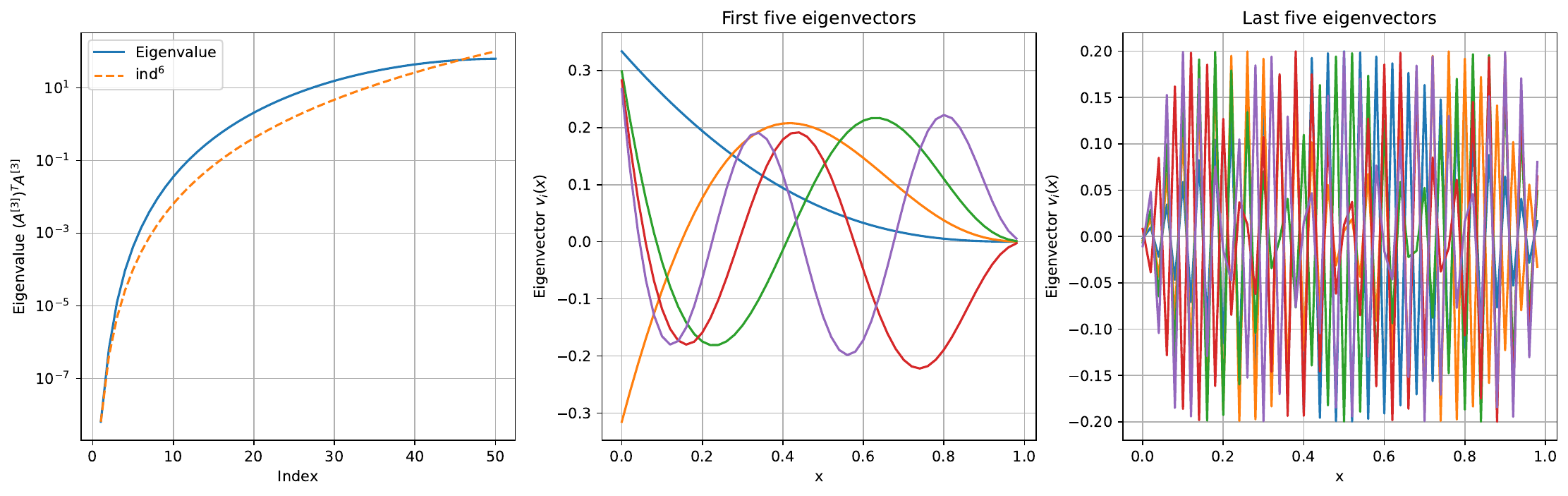}
        \caption{$r=3$}
    \end{subfigure}
    \caption{Eigenvalues of $(A^{[r]})^TA^{[r]}$ ordered smallest to largest (left) and the first (center) and last (right) five corresponding eigenvectors for $n=50$ splines knots and example orders $r\in\{1,3\}$.}
    \label{fig:cob-evec}
    \vspace{-3ex}
\end{figure}

\begin{figure}[!bht]
    \centering
    \begin{subfigure}[b]{0.475\linewidth}
        \includegraphics[width=0.8\linewidth]{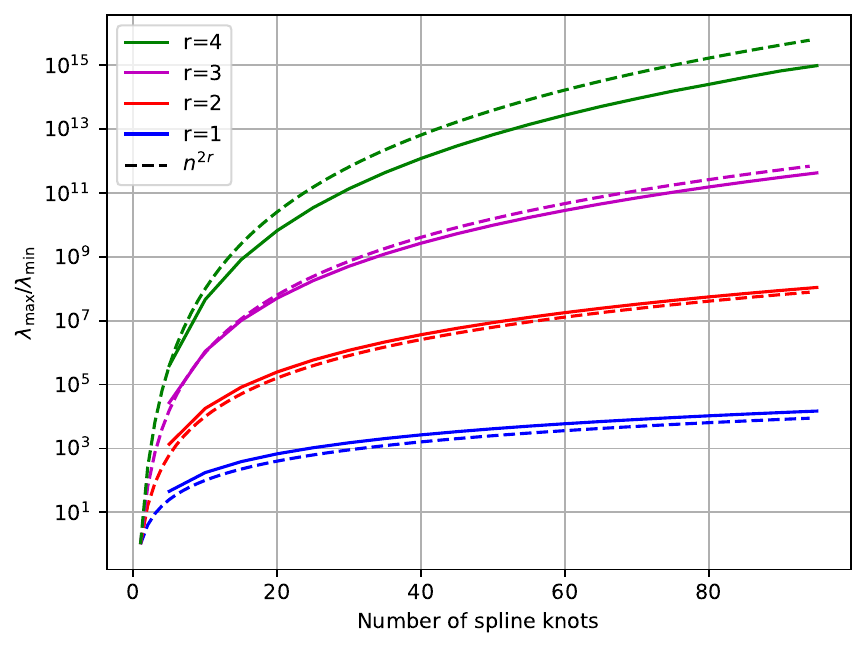}
        \caption{Ratio of largest to smallest eigenvalue of $(A^{[r]})^TA^{[r]}$ as a function of number of spline knots, shown for $r\in[1,4]$. This ratio provides the weighting of corresponding modes in gradient descent based optimization.}
        \label{fig:eval-ratio}
    \end{subfigure}
    \hspace{2ex}
    \begin{subfigure}[b]{0.475\linewidth}
        \includegraphics[width=0.8\linewidth]{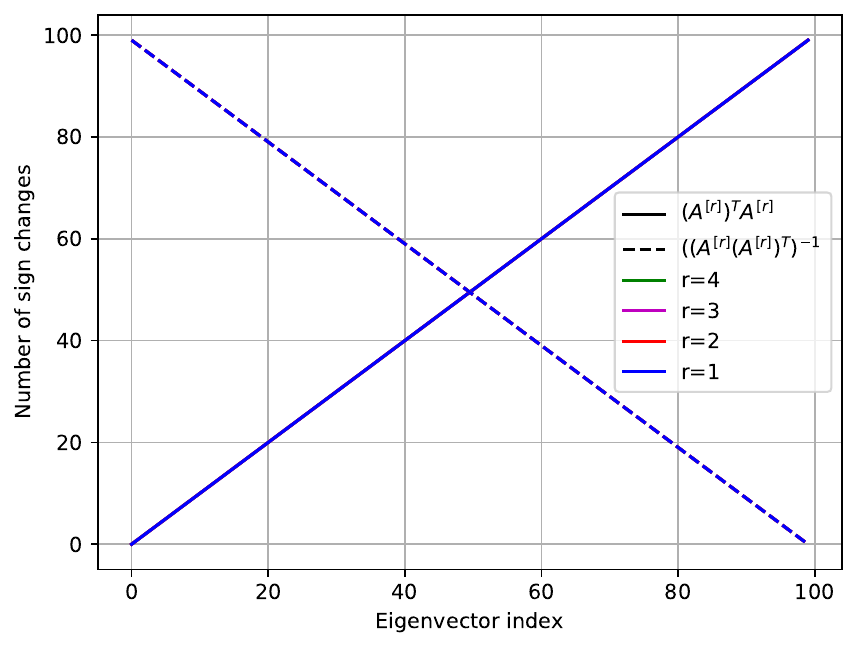}
        \caption{Number of sign changes in discrete eigenvectors as a function of eigenvector index, for $n=100$ knots, $r\in[1,4]$, and eigenvalues ordered in ascending order. Shown for $(A^{[r]})^TA^{[r]}$ and $(A^{[r]}(A^{[r]})^T)^{-1}$.}
        \label{fig:sign-change}
    \end{subfigure}
    \caption{Comparison of eigenvalue and eigenvector properties across spline orders.}
    \label{fig:eigs}
    \vspace{-5ex}
\end{figure}

We will
show in \Cref{sec:feature} that the eigen-structure of these transfer operators has important impacts on the training dynamics.

\subsection{Computational cost}\label{sec:kans-mlp:cost}
While both bases, $B_R^{[r]}$ and $B_S^{[r]}$, are equivalent in terms of their outputs, they differ in their computational cost for evaluating the output of each layer. The Cox-de Boor formula has been observed to be computationally expensive \cite{qiu2024relu}, and this reformulation using $B_R^{[r]}$ provides an impressive speedup to the underlying computational graph. Consider that for a single layer a forward-pass through the standard B-spline basis takes $O(PQ(nr+r^2))$ floating-point operations per layer, the ReLU-based formulation requires only order $O(PQ(n+r))$ operations, resulting in a speedup (in FLOPs) by a factor equal to the spline degree. This removes the need to implement different versions of the spline activations as in \cite{qiu2024relu, so2024higher}. The change-of-basis operation, i.e., multiplication by $A$, requires only $O(nr)$ operations per layer, since $A$ is a banded matrix with $r+1$ nonzero diagonals (see Appendix \ref{app:cob_proof} for explicit construction). Thus for small spline order $r$, and even moderate values of $P$ and $Q$ the cost is negligible compared to the contraction against the learnable weights. On a set of predetermined spline knots, the matrix $A$ can additionally be implemented via a convolution stencil, which is particularly efficient for uniform knots.  

To illustrate these results, we train a simple architecture with three hidden layers of width 64 using direct implementations of \eqref{eq:kan-spline-basis} and \eqref{eq:kan_relu}, where the spline basis functions are implemented via the standard Cox-de Boor recursive form. We measure the wallclock time per epoch for a range of spline orders $r$ and number of knots $n$, repeating each experiment 10 times to smooth out any noise in wallclock measurements; our wallclock measurements include only the forward and backward passes through the network, and exclude time for data loading and evaluation of a validation dataset. Linear algebra operations on modern GPUs take advantage of the hardware architectures which complicates wallclock measurements, but in Figure \ref{fig:speedup}, we still see significant improvement in wallclock time that grows with $r$.
\begin{figure}[!htb]
\centering
\includegraphics[height=0.18\paperheight]{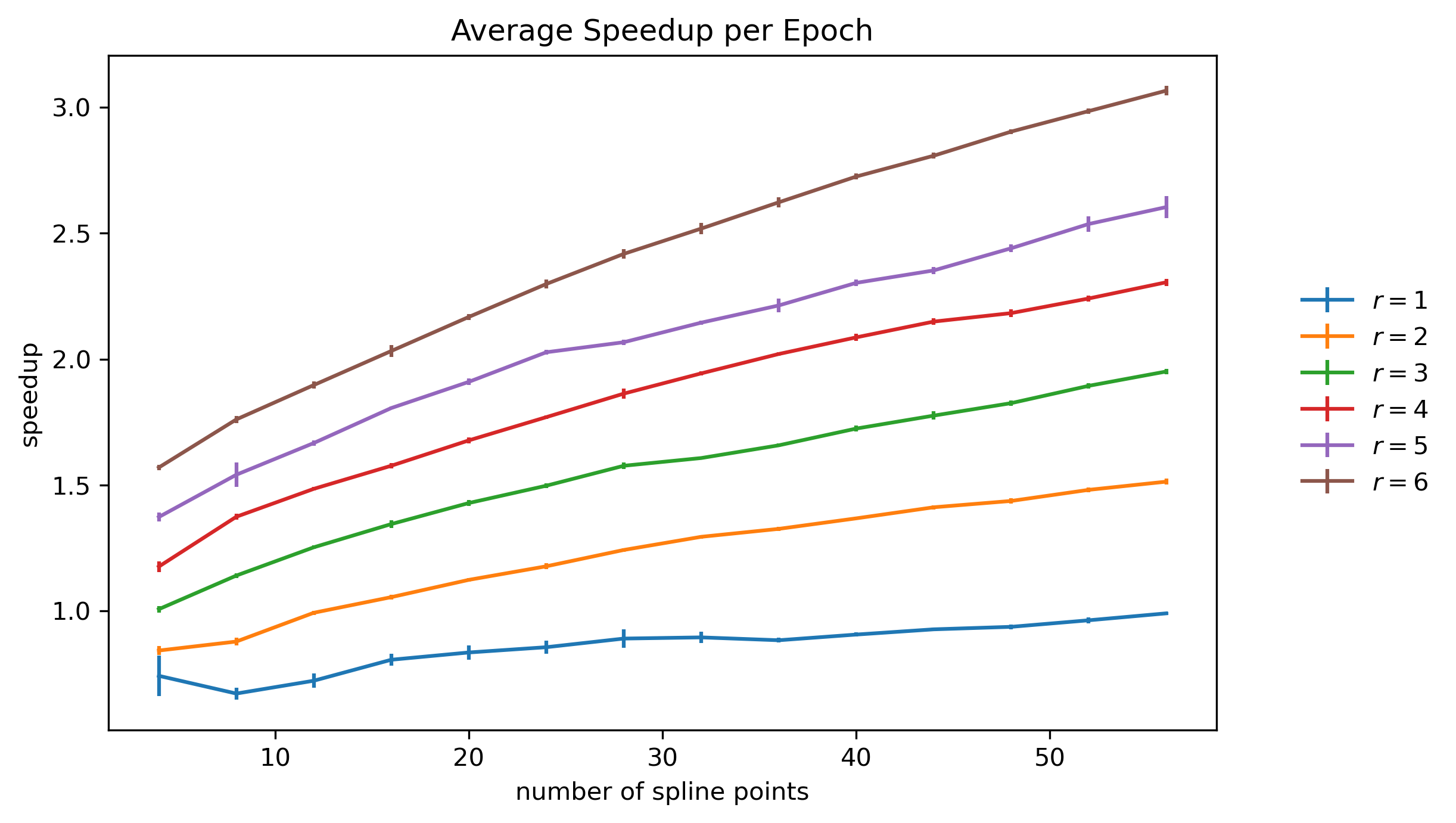}
\caption{Speedup of evaluating a layer, by applying the ReLU$^r$ activation and then the change-of-basis matrix, compared to computing the Cox-de Boor recursive formula for spline functions. Error bars show 1 standard deviation, computed over 10 instances. \label{fig:speedup} 
}
\vspace{-3ex}
\end{figure}

\section{Gradient descent and choice of basis}\label{sec:feature}

Despite an equivalence (as forward operators) under the change-of-basis, gradient-based training of KANs and multichannel MLPs yield fundamentally distinct dynamics during training, which we show by connecting the change-of-basis operations to standard results from preconditioning.

To complement the architectures in \Cref{sec:kans-mlp}, we formalize the notation for our training problem. We consider training of a neural network with input data of dimension $\din\in\mathbb{N}$ and output data of dimension $\dout\in\mathbb{N}$, where we train on a flattened batch of $k$ input-output data pairs $\{\bx,\by\}$ for $\bx\in\mathbb{R}^{k\din}, \by\in\mathbb{R}^{k\dout}$. Let $f(\bx;\bw):\mathbb{R}^{k\din}\times \mathbb{R}^{N_w} \mapsto \mathbb{R}^{k\dout}$ denote the action of the ReLU-based network on batch input data $\bx$ given weights $\bw$, and let $\L:\mathbb{R}^{k\din}\times \mathbb{R}^{k\dout}\times \mathbb{R}^{N_w} \mapsto\mathbb{R}$ be a scalar valued loss on the data for the given weights. The resulting corresponding unconstrained optimization over weights is then given by
    $\min_{\bw\in\mathbb{R}^n} \L(f(\bx;\bw),\by).$
We introduce the change of basis into this optimization problem, and define a nonlinear function $g(\bx;\bu):\mathbb{R}^{k\din}\times \mathbb{R}^{N_w} \mapsto \mathbb{R}^{k\dout}$ such that $g(\bx;\bu)= f(\bx;A^T\bu) = f(\bx;\bw)$, corresponding to a spline-based KAN. Then we can equivalently minimize
\begin{equation}\label{eq:two-loss}
    \min_{\bw\in\mathbb{R}^n} \L(f(\bx;\bw),\by) = 
        \min_{\bu\in\mathbb{R}^n} \L(g(\bx;\bu),\by).
\end{equation}
Let $\nabla_{\bw}\L\in\mathbb{R}^{N_w}$ and $\nabla_{\bu}\L\in\mathbb{R}^{N_w}$ denote the standard $\ell^2$-gradient (easily computed with back propagation) with respect to pairs $\{\bw,f\}$ or $\{\bu,g\}$ respectively. For gradient descent, we are interested in the basis representing our primal and dual spaces. If we work exclusively in the $\bw$ or $\bu$ bases for the primal and dual space, we arrive at standard potentially preconditioned gradient descent iterations:
\begin{subequations}
\begin{align}
    \bw_{k+1} & = \bw_k - \eta D_{\bw}\nabla_{\bw}\L, \label{eq:gd-w}\\
    \bu_{k+1} & = \bu_k - \eta D_{\bu}\nabla_{\bu}\L \label{eq:gd-u},
\end{align}
\end{subequations}
for linear preconditioner $D_{\bw},D_{\bu}$, e.g. from ADAM or LBFGS (or $D=I$ for gradient descent). Now consider a change of basis $\bw \mapsto A^T\bu$ or $\bu\mapsto A^{-T}\bw$. Let $J_f\coloneqq \partial f/\partial \bw \in \mathbb{R}^{k\times N_w}$ and $J_g\coloneqq \partial g/\partial\bu \in \mathbb{R}^{k\times N_w}$ denote the Jacobians of $f$ and $g$ with respect to their natural variables. Since $g(\bx;\bu) = f(\bx;A^T\bu)$ we have
    $J_g = \tfrac{\partial f}{\partial \bw} \tfrac{\partial \bw}{\partial\bu} = J_f A^T.$
Let $\nabla_m\L \in\mathbb{R}^k$ denote the derivative of $\L$ with respect to model output over $k$ data samples. From \eqref{eq:two-loss} we can compute a gradient in $\bw$-space via
\begin{equation}
    \nabla_{\bw}\L \coloneqq \left(\frac{\partial f}{\partial \bw}\right)^T\frac{\partial \L}{\partial f} = J_f^T\nabla_m \L.
\end{equation}
Similarly, we can compute a gradient in the $\bu$-space and arrive at the change of basis in the dual (gradient) space:
\begin{equation}
    \nabla_{\bu}\L \coloneqq \left(\frac{\partial g}{\partial \bu}\right)^T\frac{\partial \L}{\partial g} = J_g^T\nabla_m \L = AJ_f^T\nabla_m\L = A\nabla_{\bw}\L.
\end{equation}
Substituting $\bw \mapsto A^T\bu$ and $\nabla_{\bw}\L\mapsto A^{-1}\nabla_{\bu}\L$ into \eqref{eq:gd-w}, and $\bu\mapsto A^{-T}\bw$ and $\nabla_{\bu}\L\mapsto A\nabla_{\bw}\L$ into \eqref{eq:gd-u} we arrive at the preconditioned gradient descent iterations, respectively,
\begin{subequations}
\begin{align}
    \bu_{k+1} & = \bu_k - \eta A^{-T}D_{\bw}A^{-1} \nabla_{\bu}\L \label{eq:pgd-u}, \\
    \bw_{k+1} & = \bw_k - \eta A^TD_{\bu}A \nabla_{\bw}\L \label{eq:pgd-w},
\end{align}
\end{subequations}
where we have maintained the original preconditioners through the change of basis. This change-of-basis can also be seen as imposing the geometry of the $\bw$ space on the $\bu$ space. A change of basis by $A^T$, $\bu\mapsto A^{-T}\bu$, induces the pullback metric in $U$-space, $\langle \bx,\by\rangle_U \coloneqq \by^TAA^T\bx$, where letting $D_{\bu}=I$, the gradient is now taken with respect to the $(AA^T)$-inner product. Recall the following result \cite{Botsaris.1978}:
\begin{lemma}\label{lem:gradient-inner-product}
Consider $\L:\mathbb{R}^n\mapsto\mathbb{R}$ and let $\nabla \L(\mathbf{x})$ denote the gradient of $L$ in the $\ell^2$-inner product at $\mathbf{x}\in\mathbb{R}^n$. Let $M\in\mathbb{R}^{n\times n}$ be an SPD matrix. Then the gradient of $\L$ with respect to the $M$-induced inner product is given by 
\begin{equation}
    \nabla_M \L(\mathbf{x}) = M^{-1}\nabla \L(\mathbf{x}).
\end{equation}
\end{lemma}\
We see that \eqref{eq:pgd-u} for $D_{\bu}=I$ is gradient descent with respect to $\{\bu,g\}$ in the $(AA^T)$-inner product, arising from imposing the geometry of $W$-space on $U$-space. Similarly, for $D_{\bw}=I$ \eqref{eq:pgd-w} is equivalent to gradient descent in $\{\bw,f\}$ with respect to the $(A^TA)^{-1}$ inner product, arising from the metric pullback imposing the geometry of $U$-space on $W$-space. 

For completeness, one can also mix spaces. Suppose we iterate in $\bw$ but consider a gradient with respect to $\bu$. Such iterations would take the form $\bw_{k+1} = \bw_k - \eta D_{\bu}\nabla_{\bu}\L$. Substituting either $\nabla_{\bu}\L \mapsto A\nabla_{\bw}\L$ or $\bw \mapsto A^T\bu$ we arrive at iterations $\bw_{k+1} = \bw_k - \eta D_{\bu}A \nabla_{\bw}\L$ or $\bu_{k+1} = \bu_k - \eta A^{-T}D_{\bu}\nabla_{\bu}\L$.
Similarly we can iterate in $\bu$ but consider a gradient with respect to $\bw$. Such iterations would take the form $\bu_{k+1} = \bu_k - \eta D_{\bw}\nabla_{\bw}\L$. Substituting either $\nabla_{\bw}\L \mapsto A^{-1}\nabla_{\bu}\L$ or $\bu \mapsto A^{-T}\bw$ we arrive at iterations $\bu_{k+1} = \bu_k - \eta D_{\bw}A^{-1}\nabla_{\bu}\L$ or $\bw_{k+1} = \bw_k - \eta A^T D_{\bw}\nabla_{\bw}\L$.
Note that computing gradients in the same basis you are iterating in yields preconditioned gradient descent with SPD preconditioners related to the induced inner product. In contrast, computing a gradient with respect to a different basis then you are iterating in yields potentially nonsymmetric preconditioned descent. These methods correspond to a modified choice of duality pairing, specifying the unique mapping between every dual vector (gradient) and primal vector. The duality pairing is an invertible but not necessarily SPD matrix, in this case specifically given by the leading preconditioning operator in each equation (assuming we don't also consider a modified inner product). 

Altogether we have four distinct iterations, with equivalent realizations in $\bu$ or $\bw$ show in \Cref{tab:iterations}. Note that in the remainder of this paper we will consider iterating in the spline space $\bu$ due to its interpretability, and consider the iterations that arise from considering the $\bu$ geometry and gradient \eqref{eq:gd-u} or the $\bw$ geometry and gradient \eqref{eq:pgd-u}, with the former being the gradient descent that arises naturally in KANs. 
\begin{table}[h!]
\centering
\begin{tabular}{|c|c|c|l|}
\hline
\textbf{Geometry} & \textbf{Gradient} & \textbf{Space} & \textbf{Prec. update} \\
\hline
\multirow{ 2}{*}{$\bw$} & \multirow{ 2}{*}{$\nabla_{\bw}\L $} & $\bw$ & $\bw_{k+1} = \bw_k - \eta D_{\bw}\nabla_{\bw}\L$ \\
&& $\bu$ & $\bu_{k+1} = \bu_k - \eta A^{-T}D_{\bw}A^{-1} \nabla_{\bu}\L$ \\
\hline
\multirow{ 2}{*}{$\bw$} & \multirow{ 2}{*}{$\nabla_{\bu} \L $} & $\bw$ & $\bw_{k+1} = \bw_k - \eta D_{\bu}A \nabla_{\bw}\L$ \\
&& $\bu$ & $\bu_{k+1} = \bu_k - \eta A^{-T}D_{\bu}\nabla_{\bu}\L$ \\
\hline
\multirow{ 2}{*}{$\bu$} & \multirow{ 2}{*}{$\nabla_{\bu} \L $} & $\bw$ & $\bw_{k+1} = \bw_k - \eta A^TD_{\bu}A \nabla_{\bw}\L$ \\
&& $\bu$ & $\bu_{k+1} = \bu_k - \eta D_{\bu} \nabla_{\bu}\L$ \\
\hline
\multirow{ 2}{*}{$\bu$} & \multirow{ 2}{*}{$\nabla_{\bw} \L $} & $\bw$ & $\bw_{k+1} = \bw_k - \eta A^T D_{\bw}\nabla_{\bw}\L$ \\
&& $\bu$ & $\bu_{k+1} = \bu_k - \eta D_{\bw}A^{-1}\nabla_{\bu}\L$ \\
\hline
\end{tabular}
\caption{Preconditioned gradient descent update rules under different parameterizations and gradient representations. We associate the initial preconditioner to be consistent with the gradient, e.g. $\nabla_{\bu}$ has preconditioner $D_{\bu}$.}
\label{tab:iterations}
\vspace{-4ex}
\end{table}

The preconditioning induced by the change of basis coupled with the spectral theory from \Cref{sec:feature:eval} provides a framework for understanding and ensuring \emph{complementary} optimization or ``relaxation'' on each level in the multilevel hierarchy proposed in the next section (see \Cref{sec:multilevel:relax}). 

\begin{remark}
    Given the structure of $A$ and each block $A^{[r]}$ derived in \Cref{sec:feature}, we can bound the spectral radius of the neural tangent kernel (NTK) to show that gradient based optimization routines for a ReLU and spline basis as derived above can be used with comparable learning rates, despite the preconditioning induced by the change of basis. However, the resulting training dynamics will differ significantly; for more on the NTK of these matrices, see the Supplementary Material.
\end{remark}

\section{Multilevel training of KANs}\label{sec:multilevel}

We build upon the change-of-basis and preconditioning results for KANs to develop an efficient multilevel training framework for spline-based KANs, exploiting the additional structure that comes with the spline parameterization.
With the preconditioning results above, we can view multichannel MLPs in a natural spline basis that provides a functional framework for building models with straightforward hierarchy of scales and transfer operators. KANs with spline basis functions and geometric transfer operators provide properly nested hierarchies of different refinement and complementary level-specific optimization. 

\subsection{General multilevel formulation}

 Let $\CP : \mathbb{R}^{N_c}\mapsto\mathbb{R}^{N_f}$ be a linear interpolation and change of basis from coarse weight space of dimension $N_c$ to fine weight space of dimension $N_f$. Optimizing over coarse space $\bu^{(c)}$, we have gradient updates
\begin{equation}
    \bu^{(c)}_{k+1} = \bu^{(c)}_k + \eta \CP^T \nabla \L(\CP\bu^{(c)}_k).
\end{equation}
Suppose $\nabla\L(\bu) = L\bu$ for linear operator $L$. Then this reduces to
\begin{equation}
    \bu^{(c)}_{k+1} = \bu^{(c)}_k + \eta \CP^T L\CP \bu^{(c)}_k,
\end{equation}
which is exactly a Richardson iteration on a Galerkin coarse grid operator $L_c \coloneqq \CP^TL\CP$. If we further include a diagonal scaling in $\CP\bu^{(c)}$ such that $L_c$ has unit diagonal, this results in a Jacobi relaxation iteration applied to the Galerkin coarse grid operator, which is exactly the computational kernel in algebraic multigrid coarse grid correction.\footnote{Petrov-Galerkin transfer operators with restriction $\CR\neq \CP^T$ can be achieved by considering a gradient in a non-standard inner product or duality pairing. We will not consider such cases here.}

Returning to optimization, consider model $g(\bx;\bu):\mathbb{R}^{k\din}\times \mathbb{R}^{N_w} \mapsto \mathbb{R}^{k\dout}$ and loss $\L:\mathbb{R}^{k\din}\times \mathbb{R}^{k\dout}\times \mathbb{R}^{N_w} \mapsto\mathbb{R}$. To ensure good approximation of the fine-level objective by coarse-level updates, a natural approach would be defining a coarse subspace optimization problem via
\begin{equation}\label{eq:coarse-opt}
    \min_{\bu^{(c)}\in\mathbb{R}^{N_c}} \L(g(\bx;\CP\bu^{(c)}),\by).
\end{equation}
This interpretation provides a framework to coarsen in weight space while ensuring good coarse approximation (a key component of multigrid almost all multilevel machine learning methods currently lack). In practice we do not want every level to be as expensive to evaluate as the finest level, as is naively the case in \eqref{eq:coarse-opt} due to the size of the inner model $g$ through which gradients are back-propagated. This does provide a natural target for designing multilevel hierarchies and considering approximation properties though. Consider a two-level hierarchy. Following the definition in \eqref{eq:coarse-opt}, we define a natural relation between coarse and fine architectures and corresponding transfer operators for a hierarchy that is properly nested in terms of functional approximation. 
\begin{definition}[Properly nested hierarchy] \label{def:prop-nested}
Let $g_f(\bx; \bu^{(f)}):\mathbb{R}^{k\din}\times \mathbb{R}^{N_f}\mapsto\mathbb{R}^{\dout}$ denote the action of the fine operator on flattened input vector $\bx\in \mathbb{R}^{k\din}$ with fine-grid weights $\bu^{(f)}\in\mathbb{R}^{N_f}$ and $g_c(\bx; \bu^{(c)}):\mathbb{R}^{k\din}\times \mathbb{R}^{N_c}\mapsto\mathbb{R}^{\dout}$ denote the action of the coarse operator on flattened input vector $\bx\in \mathbb{R}^{k\din}$ with coarse-grid weights $\bu^{(c)}\in\mathbb{R}^{N_c}$. Let $\CP : \mathbb{R}^{N_c}\mapsto\mathbb{R}^{N_f}$ be a linear interpolation operator from coarse to fine weight space. We define $\{g_f,g_c,\CP\}$ as a \emph{properly nested hierarchy} if
\begin{equation}\label{eq:nested}
    g_c(\bx; \bu^{(c)}) = g_f(\bx; \CP\bu^{(c)}),
\end{equation}
that is, a properly nested hierarchy imposes that the fine operator exactly preserves the action of the coarse operator under interpolation of weights.
\end{definition}
This definition guarantees that interpolation of weights does not undo progress made via coarse level optimization, as the fine model and corresponding loss will match that of the coarse model exactly. If we satisfy the above definition, we can construct a change-of-basis hierarchy without having to evaluate a network with fine level cost on every level. In particular, we satisfy the following proposition. 

\begin{proposition}[Subspace optimization] \label{prop:coarse-opt}
Let $g_f,g_c:\mathbb{R}^{\din}\mapsto\mathbb{R}^{\dout}$ and $\CP : \mathbb{R}^{N_c}\mapsto\mathbb{R}^{N_f}$ be as in \Cref{def:prop-nested}, and define fine and coarse loss functions $\L:\mathbb{R}^{k\din}\times \mathbb{R}^{k\dout}\times \mathbb{R}^{N_f} \mapsto\mathbb{R},\L_c:\mathbb{R}^{k\din}\times \mathbb{R}^{k\dout}\times \mathbb{R}^{N_c} \mapsto\mathbb{R}$ that are identical with respect to model output, that is, if $g_c(\bx; \bu^{(c)}) = g_f(\bx; \bu^{(f)})$ then $\L_c(g_c(\bx; \bu^{(c)}),\by) = \L(g_f(\bx; \bu^{(f)}),\by)$. Then
\begin{equation}
    \min_{\bu^{(c)}} \L_c(g_c(\bx;\bu^{(c)}),\by) = \min_{\bu^{(c)}} \L(g_f(\bx;\CP\bu^{(c)}),\by).
\end{equation}
\end{proposition}
This is a simple result that formalizes the purpose of the properly nested hierarchy -- optimizing in the coarse space with coarse loss $\L_c$ is equivalent to optimizing the fine loss with weights restricted to a subspace. However, we must be careful with our coarsening to ensure \eqref{eq:nested} is feasible. If we take something like coarsening-in-layer, as considered a number of times in the literature, e.g. \cite{Gunther2020,cyr2025torchbraid,Kopanicakova2022}, such a constraint may not be possible to satisfy. For an arbitrary number of levels, we define our nested multilevel optimization in \Cref{alg:mgopt}. Note, in this paper we do not consider proper multilevel cycling. Nonlinear multigrid and multigrid optimization are notoriously sensitive, and will be a topic for future work. 

\begin{algorithm}[!htb]
\caption{Nested Multilevel Optimization}
\label{alg:mgopt}
\begin{algorithmic}[1]
\REQUIRE Hierarchy $\{g_k\}_{k=1}^K$, transfer operators $\{\CP_{k}^{k-1}\}_{k=2}^K$, training data $\bx$
\ENSURE Solution $\bu_1$ to finest level problem, $\min_{\bu_1} \L(g_1(\bx;\bu_1))$.

\State{Initialize $\bu_K$} \hfill\COMMENT{Initialize weights on coarsest level}
\FOR{$k=K$ to $2$:}
    \STATE Solve $\min_{\bu_k} \L(g_k(\bx; \bu_k))$
    \hfill\COMMENT{Solve level $k$ optimization with initial guess $\bu_k$}
    \STATE $\bu_{k-1} \gets \CP_{k}^{k-1}\bu_\ell$ 
    \hfill\COMMENT{Interpolate solution to finer level as initial guess}
\ENDFOR
\STATE Solve $\min_{\bu_1} \L(g_1(\bx; \bu_1))$
\hfill\COMMENT{Solve optimization on finest level with initial guess $\bu_1$}

\State {\bf return} $\bu_1$
\end{algorithmic}
\end{algorithm}
\subsection{Transfer operators in KANs}\label{sec:multilevel:relax}

Using multiresolution features of splines to build KANs at differing resolutions was introduced in \cite{liu2024kan}. However, the refinement done in \cite{liu2024kan} does not perform geometric refinement; instead, they interpolate between grids of arbitrary sizes. To do so, they solve for each hidden node a least-squares problem of size $N_\text{data} \times (n+r-1)$, which is (unsurprisingly) prohibitively costly. This idea was dismissed by  the authors as being too slow for practical applications.
In contrast, it is substantially more efficient to refine geometrically and then use restriction and prolongation operators from multigrid literature \cite{hollig2003finite, hackbusch2013multi} to perform the grid transfer, thus enabling multilevel training to become much more feasible and cost-effective.

We first make the observation that for two sets of knots $T,\,T'$ where $T \subset T'$, we have $S_r(T') \subset S_r(T)$. Thus there exists an interpolation operator $\CP$ from one grid to another in the case of nested grids.
The splines defined by an arbitrary set of knots $T$ on arbitrary coarse grid with basis $\left\{b_{i,T}^{[r]}\right\}_{i=1-r}^{n-1}$
has a corresponding basis with knots $T'$ on a finer grid given by $\left\{ b_{i,T'}^{[r]} \right\}_{i=1-r}^{m-1}$ with $m>n$.
For a KAN layer with knots $T$, we can therefore define a fine KAN layer with knots $T'$ whose values are identical, since for an interpolation matrix $\CP$, we have
\begin{equation} \begin{split}
    x_q^{(\ell+1)} &= \sum_{p=1}^P \sum_{i=1-r}^{n-1} \widetilde{W}_{qpi,T}^{(\ell)} \, b_{i,T}^{[r]}(x_p^{(\ell)}) 
    =\sum_{p=1}^P \sum_{i=1-r}^{n-1}
    \sum_{j=1-r}^{m-1}  \widetilde{W}_{qpi,T}^{(\ell)} \, \CP_{ij} \, b_{j,T'}^{[r]}(x_p^{(\ell)}) \\
    &\qquad := \sum_{p=1}^P \sum_{j=1-r}^{m-1} \widetilde{W}_{qpj,T'}^{(\ell)} \, b_{j,T'}^{[r]}(x_p^{(\ell)}).
\label{eq:kan-interpolation}
\end{split} \end{equation}
In the general case of arbitrary $T \subset T'$, using the respective change-of-basis matrices $A^{[r]}_T,\,A^{[r]}_{T'}$ defined via Equation~\eqref{eq:A_recurrence}, we can write 
\begin{equation*} 
\CP = A^{[r]}_{T} \, \mathcal{I} \, {A^{[r]}_{T'}}^{-1}, \quad \text{where} \quad \mathcal{I}_{ij} = \begin{cases} 1 &  t_i = t'_j \text{ for } t_i \in T,\, t'_j \in T' \\ 0 & \text{else} \end{cases} .
\end{equation*}
In the case of dyadic refinement on a uniform spline grid, $\CP$ is explicitly given \cite{hollig2003finite} as
\begin{equation}
\CP_{ij} = \begin{cases} 2^{-r} \begin{pmatrix} r \\ s\end{pmatrix} & j=2i+s \\ 0 & \text{else}\end{cases} .
\end{equation}
The refinement process described in Equation~\eqref{eq:kan-interpolation}
produces a new KAN layer that is functionally equal to the previous, but with more trainable parameters. Under this definition, we satisfy \Cref{def:prop-nested}, and can thus define a properly nested hierarchy \eqref{eq:coarse-opt} without having to evaluate a network with fine level cost on every level. This is one of the fundamental features of multilevel KANs, as it means the refinement process does not modify the training progress made on coarse models. 

Thus far, the multilevel method for KANs is not unique to our B-spline approach. The transfer operators and properly nested hierarchy are valid for broad classes of KANs where activation functions are discretized. Any discretization of activation functions with a similar nesting of spaces upon grid (or polynomial) refinement satisfies this same result, which we will state below.

\begin{lemma} \label{lemma:multilevelkans}
Any KAN with activations defined over a basis which is nested under a refinement procedure
can be supplied with interpolation operators that yield a properly nested hierarchy (\Cref{def:prop-nested}).
\end{lemma}
\begin{proof}
Let $\mathcal{B} = \{\phi_i:[a,b]\to\mathbb{R}\}_{i=1}^{n_B}$ be a basis of size $n_B$
for the activation function space (e.g. B-splines, Chebyshev, wavelets)
on a domain $[a,b]$ with $a,b\in\mathbb{R}$, $a<b$,
and let $\mathcal{B}' = \{\phi_i':[a,b]\to\mathbb{R}\}_{i=1}^{n_{B'}}$
be a basis of size $n_{B'}>n_B$ such that $\mathcal{B}' \supset \mathcal{B}$.
Then there is an operator $\CP_{\cal B}^{\cal B'}:\mathcal{B}\to \mathcal{B}'$ which interpolates
basis $\mathcal{B}$ to $\mathcal{B}'$ such that for $\sigma\in \mathcal{B}$
we have $\CP_{\cal B}^{\cal B'}\sigma\in \mathcal{B}'$
and $\sigma(x) \equiv \CP_{\cal B}^{\cal B'}\sigma(x)$.
\end{proof}

Note that this form of network refinement is tied to the discretization of the activation functions; essentially this performs refinement with respect to the channel dimension of the weight tensors, instead of refining the number of hidden nodes at each layer. While there are efforts to refine networks by increasing the number of layers or hidden nodes, the method in Lemma \ref{lemma:multilevelkans} does not change the number of nodes in the network.

\subsection{Complementary relaxation}

Complementary processes to attenuate different error modes is arguably the most fundamental component of a successful multigrid method, and also the most challenging for multilevel ML. Previously we discussed how KANs enable a well-defined nested hierarchy of ML models. In addition, it is critical that the ``relaxation'' or local training on different levels is complementary, in that when you interpolate your model to a finer level, your optimization takes advantage of this new expressivity.

From the results in \Cref{lem:deriv} and \Cref{tab:iterations}, we can impose a strong preference in descent-based optimization of multichannel MLPs/KANs toward learning higher or lower frequency functions with respect to the spline knots based on the geometry of the gradient and objective space. For simplicity consider gradient descent, $D_{\bu}, D_{\bw}=I$. Recall from \eqref{eq:pgd-u}, optimizing in the ReLU geometry and basis imposes a preconditioning in the spline basis. For stability, the learning rate must be chosen to normalize that the largest eigenmode of the preconditioned operator. When the preconditioner $(AA^T)^{-1}$ has a large spectral range, gradients associated with small eigenvalues of $(AA^T)^{-1}$ provide almost no correction to the weights. As discussed in \Cref{sec:feature:eval}, the eigenfunctions of even-order differential operators are ordered from smoothest to most oscillatory in terms of magnitude from smallest to largest, and $(AA^T)^{-1} \sim (\partial_x)^{-r}$ will be opposite. Moreover, the eigenvalues span a range of magnitudes. For $r=1$, we have $(AA^T)^{-1}\sim \Delta^{-1}$, where the smoothest modes are order $\mathcal{O}(1)$ and oscillatory modes on the spline grid are are order $\mathcal{O}(h)^2$ for knot spacing $h$.  Thus gradient based optimization with a basis induced preconditioning $(AA^T)^{-1}$ will weight gradient corrections that are smooth in knot space by orders of magnitude more than oscillatory ones. This is especially true for $r>1$. Thus multichannel MLPs do \emph{not} satisfy a complementary relaxation, as \emph{all} levels will strongly favor optimizing smooth functions in the knot space. This is the opposite of what we want, as after refining our spline knots we have added new expressivity via more complex/less smooth functions. If the optimizer cannot quickly optimize over this new space, the refinement will be a waste. This phenomenon will be demonstrated numerically in \Cref{sec:results}. 

In contrast, in the natural KANs spline basis the free weights being optimized correspond to local basis functions with compact support. As a result, a nonzero gradient at a specific weight results in an update that largely effects a local region centered at the corresponding spline knot within the larger KANs functional composition. To that end, we expect the gradient to transfer directly to these coefficients and naturally support oscillatory functions on the resolution of spline knots, specifically due to the compact support and localization. This is exactly the behavior needed in a multilevel training setting, so that when we perform a uniform refinement of our spline knots, gradient-based optimizers on the refined model immediately start utilizing the new expressivity in reducing the loss. Indeed, this behavior is exactly what we will demonstrate numerically in the following section. 

\section{Numerical results}\label{sec:results}

\subsection{Regression}\label{sec:results:regression}
First we consider function regression of a 0.175 radian counterclockwise ($\approx 10.03^\circ$) coordinate rotation of the nonsmooth function
\begin{equation}
f(x,y) = \cos(4\pi x) + \sin(\pi y) + \sin(2\pi y) + |\sin(3\pi y^2)|,
\label{eq:nonsmooth-regression}
\end{equation}
with a $[2,5,1]$ architecture.
Optimization is performed with geometry and gradient defined by the spline basis ($\bu$ in the context of \Cref{sec:feature}) with descent \eqref{eq:gd-u}, and the ReLU basis ($\bw$ in the context of \Cref{sec:feature}) with descent \eqref{eq:pgd-u}. In both cases, we consider training a coarse model, a fine model, and a multilevel training schedule, as well as a comparable vanilla MLP architecture. We train using L-BFGS, each with equivalent amounts of work (FLOPs) during training. Epochs counts are denoted in a list such as $\{32,16,8,4\}$, meaning 32 epochs are performed on the initial model, and then 16 epochs are performed after transferring to a refined model where the original grid was evenly subdivided once, and so on. The \textit{coarse model} corresponds to $\{128,0,0,0\}$, the \textit{fine model} to $\{0,0,0,16\}$, and the \textit{multilevel model} as $\{32,16,8,4\}$.  Results are shown in \Cref{table:ml-results-regression}
averaged across five random initializations. 

\begin{table}[!ht]
\caption{
Accuracy for regression trained under different bases, architectures, and training regimes. Standard deviations are computed over random initializations ($N=5$).}
\label{table:ml-results-regression}
\vskip 0.15in
\begin{center}
\begin{small}
\begin{sc}
\begin{tabular}{lcllccc}
\toprule
Type & Layers & Basis & Fidelity &\# Param. &  MSE: mean (stdev) \\
\midrule
KAN & $[2,5,1]$ & ReLU & coarse     &  55 & $1.10 \times 10^{-2} \, (8.32 \times 10^{-3})$\\
KAN & $[2,5,1]$ & ReLU & fine       &230 & $1.35 \times 10^{-0} \, (1.50 \times 10^{-1})$\\
KAN & $[2,5,1]$ & ReLU & multilevel &230 & $1.06 \times 10^{-2} \, (8.03 \times 10^{-3})$\\
\midrule
KAN & $[2,5,1]$ & Spline & coarse     & 55 & $1.65 \times 10^{-3} \, (4.18 \times 10^{-5})$\\
KAN & $[2,5,1]$ & Spline & fine       &230 & $2.54 \times 10^{-3} \, (5.64 \times 10^{-3})$\\
KAN & $[2,5,1]$ & Spline & multilevel &230 & $\mathbf{ 3.67 \times 10^{-5}} \, (7.19 \times 10^{-5})  $\\
\midrule
MLP & $[2,5,1]$     & ReLU &&  20 &   $2.94\times 10^{-2} \, (1.60 \times 10^{-2})$\\
MLP & $[2,30,1]$    & ReLU && 120 &   $1.02 \times 10^{-3} \, (9.42 \times 10^{-4})$\\
MLP & $[2,20,20,1]$ & ReLU && 500 &   $3.33 \times 10^{-4} \, (2.92 \times 10^{-4})$ \\
\bottomrule
\end{tabular}
\end{sc}
\end{small}
\end{center}
\vskip -0.1in
\end{table}

First and foremost, we see that training in the spline basis $\bu$ leads to $1-3$ orders of magnitude improvement in accuracy compared with the ReLU basis $\bw$. It is also clear that for fixed number of training epochs, the fine model, which is in principle more expressive, reaches at best a comparable accuracy to the coarse model for spline basis, and two orders of magnitude \emph{worse} for the ReLU basis. Multilevel training in the spline basis significantly accelerates training, obtains orders of magnitude better accuracy than just training on a coarse or fine model, and an order of magnitude better than a larger vanilla MLP. In contrast, multilevel training in the ReLU basis provides effectively no improvement over just the coarse model. Building on theory from \Cref{sec:feature}, this makes perfect sense. When we refine our spline knots and begin gradient-based training of the finer model in a ReLU basis, the associated preconditioning strongly favors geometrically smooth modes by orders of magnitude. However, geometrically smooth modes are largely able to be represented by the coarse model in the first place, which leads to the zero improvement in accuracy when training the refined model in the ReLU basis. This is an important observation in the context of multilevel machine learning, and effectively an extension of multigrid approximation properties to the context of stochastic optimization -- it is critical that training on the fine model is complementary to training on the coarse model.

Convergence histories under the same amount of training work are shown in \Cref{fig:convergence-history} to illustrate the impact of refinement on training history. The points where the multilevel model loss stagnates before abruptly dropping correspond to the epochs where the model is refined. Note that training on the fine models (KAN or MLP) have effectively stalled, with loss decreasing extremely slowly. This emphasizes the value of the multilevel training, where the loss evolution plot demonstrates that standalone fine KAN or MLP models would take significantly more epochs to reach the accuracy obtained by the multilevel training, if at all.
\begin{figure}[!htb]
  \centering
    \centering
    \includegraphics[height=1.45in,clip,trim=0 0 80mm 0]{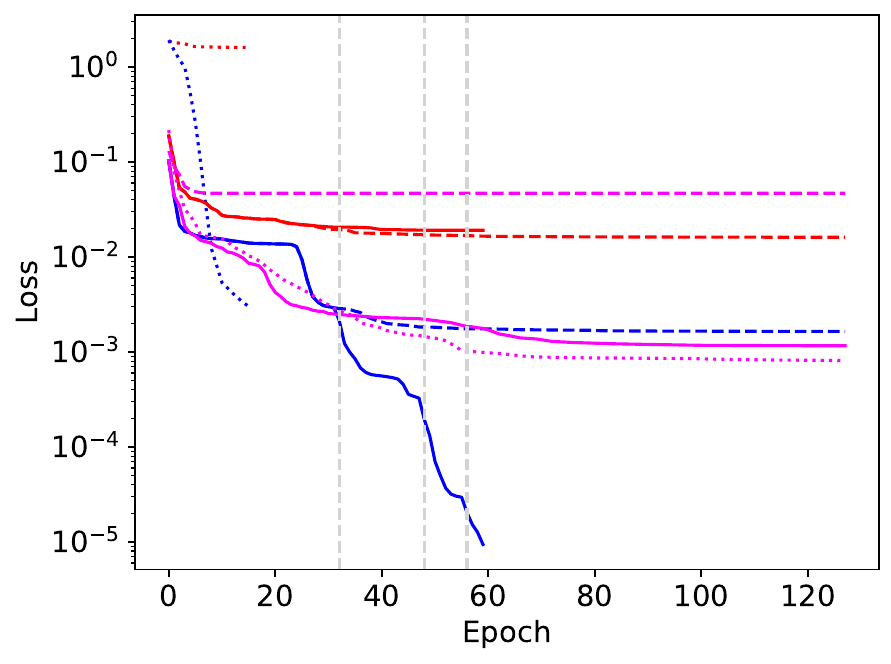}
    \includegraphics[height=1.45in,clip,trim=150mm 0 0 0]{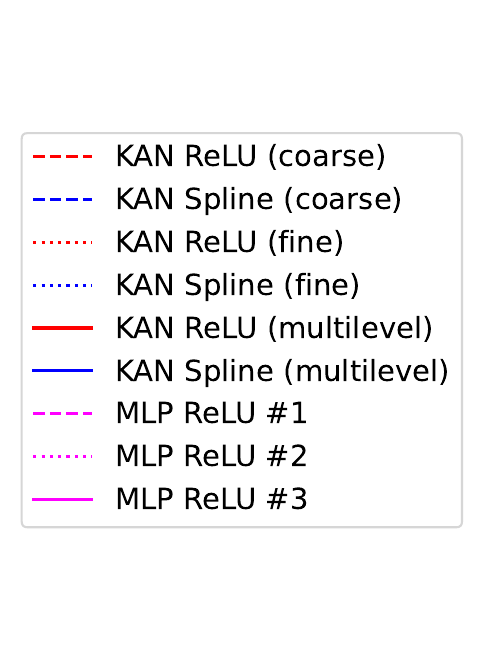}
  \caption{Select convergence history for regression under approximately same amount of work for all different models. Vertical lines indicate refinements for multilevel models.}
  \label{fig:convergence-history}
\end{figure}

\subsection{Physics-informed neural networks}\label{sec:results:pinns}

This section explores the impact of multilevel training physics-informed neural networks (PINNs). PINNs is by now established as a leading neural network approach for approximating the solution of PDEs. 
Training these models can be difficult with a number of tricks required to achieve reasonable solutions efficiently~\cite{wang2023expert,wang2025simulating}. Our goal is to use PINNs to study the performance characteristics of our KANs methodologies compared to traditional architectures. To that end, we limit the use of specialized data modifications and training techniques that yield state-of-the-art PINNs solutions. One \emph{included} enhancement is residual-based attention (RBA), proposed in~\cite{anagnostopoulos2024}, applied to the volumetric loss term. RBA improved convergence without obscuring the interpretation of the results.

\subsubsection{2d Poisson} \label{sec:results:pinns:poisson}
The first PINN considered solves the 2d Poisson equation:
\begin{equation}
\begin{split}
\nabla\cdot\left(\epsilon(\vec{x}) \nabla u(\vec{x})\right) &= f(\vec{x}) \quad \vec{x} \in [-1,1]^2, \\
u(\vec{x}) &= 0 \quad x \in\partial\Omega.
\end{split}
\end{equation}
Letting $\vec{x}=[x,y]$, the coefficient and the manufactured solution are
\begin{equation}
\epsilon(\vec{x}) =
\left\{
\begin{array}{cc}
\epsilon_l = 1 & x < 0 \\
\epsilon_r = \frac12 & x \geq 0
\end{array}
\right., \;\;
u^*(\vec{x}) = \left\{
\begin{array}{cc}
\sin(\pi x)\sin(3\pi y)& x < 0 \\
\sin(2\pi x)\sin(3\pi y) & x \geq 0
\end{array}
\right.
\end{equation}
with the appropriate forcing $f(\vec{x})$. The PINN loss shown in \Cref{fig:2d-poisson-details} (left) is
defined over three sets of equally spaced points; the volume $V$ ($2401$ interior points on the interior), the boundary $B$ ($200$ points on the exterior) and the interface $I$ ($49$ points where $x=0$).
The $\mathcal{L}_I$ terms enforces the correct jump in the gradient at $x=0$. To permit the use of a single neural network approximation (as opposed to separate networks in each subdomain) we augment the network with an additional level set field as proposed in~\cite{tseng2023cusp}.

\begin{figure}[htp]
\noindent
  \begin{minipage}{0.48\textwidth}
    \scriptsize
    \begin{equation*}
    \begin{split}
      \mathcal{L}_V(\theta)
        &= \sum_{\vec x \in V} \bigl|\nabla\!\cdot\!(\epsilon(\vec x)\,\nabla u_N(\vec x))
           \!-\! f(\vec x)\bigr|^2\\
      \mathcal{L}_B(\theta)
        &= \sum_{\vec x \in B} \bigl|u_N(\vec x)\bigr|^2\\
      \mathcal{L}_I(\theta)
        &= \sum_{\vec x \in I} 
           \bigl|(\epsilon_l\!\nabla u_N(\vec x_l)
                 \!-\! \epsilon_r
                 \!\nabla u_N(\vec x_r))\!\cdot\![1,0]\bigr|^2\\
      \mathcal{L}(\theta)
        &= \gamma_V\,\mathcal{L}_V(\theta)
         +  \gamma_B\,\mathcal{L}_B(\theta)
         +  \gamma_I\,\mathcal{L}_I(\theta)
    \end{split}
    \end{equation*}
    \vspace{1mm}
  \end{minipage}
  \hfill
  \begin{minipage}{0.5\textwidth}
    \begin{small}
    \begin{sc}
    \addtolength{\tabcolsep}{-0.4em}
    \begin{tabular}{lclc}
      \toprule
      Type & Layers      & Basis  & \# Param. \\
      \midrule
      KAN  & [2,5,1]   & Spline & 140,\,220,\,380,\,700 \\
      KAN  & [2,5,1]   & ReLU   & 140,\,220,\,380,\,700 \\
      MLP  & [2,16,16,1] & ReLU   & 352                 \\
      \bottomrule
    \end{tabular}
    \addtolength{\tabcolsep}{0.4em}
    \end{sc}
    \end{small}
  \end{minipage}
  \caption{For the 2d Possion problem in Sec.~\ref{sec:results:pinns:poisson} (left) the  volume, boundary and interface components of the PINN loss, and (right) multilevel KAN and MLP architectures and parameter counts. 
  The neural network $u_N$ has an assumed dependence on parameters $\theta$.
  }
  \label{fig:2d-poisson-details}
\end{figure}

\Cref{fig:2d-poisson-details} (right) shows three different networks studied. Both KAN architectures use a multilevel training, with the coarsest level containing $4$ splines. The number of splines doubles at each subsequent finer level. The networks are trained using Adam. Additional hyperparameters including loss term weightings can be found in the Supplemental Material. %
The loss and error as a function of epoch are shown in \Cref{fig:pinns} for KAN with a spline basis (left), KAN with a ReLU basis (middle), and vanilla MLPs (right). The vertical light gray dashed lines indicate an increase in the number of splines as the
multilevel algorithm changes levels.
Both the multilevel KAN with spline basis and the vanilla MLP achieve small loss and small error. Multilevel training with the ReLU basis fails to make progress however, and stagnates slightly below $\mathcal{O}(1)$ relative $\ell^2$-error. Note that as the spline case, and the ReLU case form the same approximation space, the training dynamics alone account for this difference. We hypothesize that the good performance of the spline basis is a result of learning progressively higher frequency modes, as suggested by the theory developed in \Cref{sec:feature}.

We also see the substantial advantages provided by multilevel training. The multilevel KAN with a spline basis achieves smaller relative error faster than a comparable MLP, with significantly less noise in the actual error with respect to epoch. This latter point is particularly important for robustness and generalization when the target solution is \emph{not} known. 
Returning to the multilevel KAN training, we also emphasize the clear stair-casing effect in the multilevel training. For each refinement, particularly the refinements at 1250 and 2500 epochs, the loss and error have largely plateaued before refinement, and immediately after refinement we see rapid decrease in loss and error, demonstrating how the model immediately takes advantages of its increased expressivity. This behavior is indicative of a successful multilevel hierarchy and training methodology, and also something we have not seen anywhere in the literature studying multilevel methods for machine learning. 
\begin{figure}[!htb]
    \centering
    \begin{subfigure}{\textwidth}
        \centering
        \includegraphics[width=\textwidth]{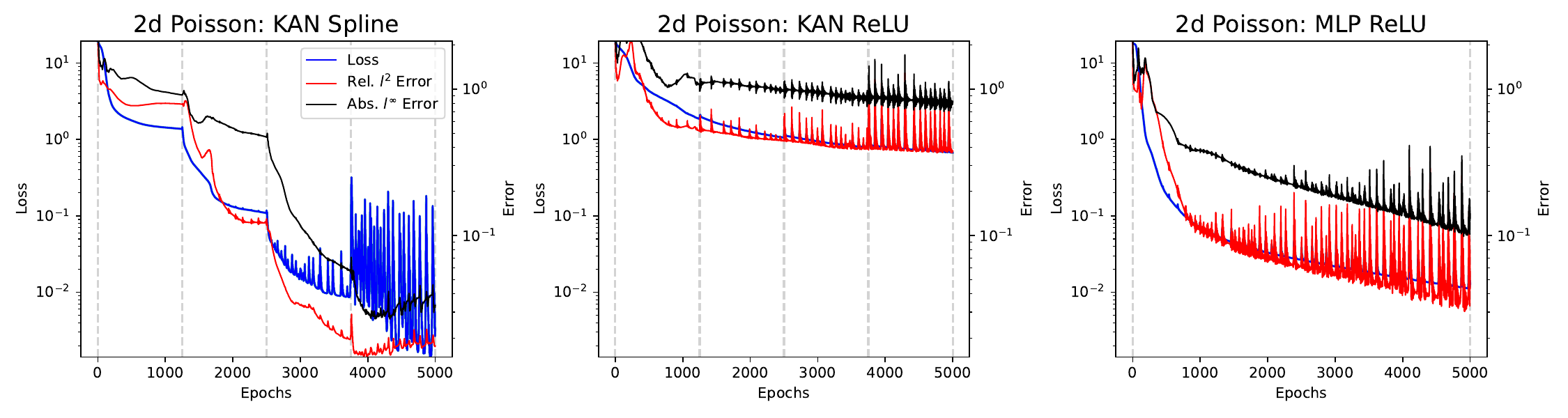}
    \end{subfigure}
    \caption{2d Poisson: Plots of loss and error vs. epoch for multilevel training of KANs under a spline basis (left), multilevel training of KANs under a ReLU basis (center), and training of comparable MLP (right). The KAN Spline and KAN ReLU are identical in terms of approximation power, yet differences in multilevel training performance suggest that the relaxation for KAN Spline complements the refinement strategy.}
    \label{fig:pinns}
\end{figure}

\subsubsection{1d Burger's} \label{sec:results:pinns:burgers}

Next we consider a short example training a PINN for 1d Burger's equation $u_t + u u_x - \nu u_{xx} = 0$ with $\nu = 10^{-2}$ as a physical problem with inherently lower regularity. The solution is represented on space-time collocation points on a $64\times 64$ grid, and the entire batch is used each optimization step.
We use weight-regularized Adam and additionally employ a exponentially-cyclic learning rate scheduler as done in other places, e.g. \cite{wang2023expert}. We run for 10000 steps for the basis results; for refinement, we use $[3200,0,0]$ epochs for our coarse model, $[0,0,800]$ for the fine model (one level coarser than the regression problems), and $[800,400,200]$ for the multilevel model. Given the poor performance of the ReLU basis demonstrated previously, we only show results using the spline basis for this problem. Results are shown in Table \ref{table:results-pinn}. The KAN architectures - even the coarsest ones - outperform comparable MLP architectures in terms of training loss. Multilevel training is where we see significant improvement, obtaining 2--3 orders of magnitude better accuracy than stand alone KANs or MLPs on fine or coarse architectures, for comparable amounts of work.

\begin{table}[!htb]
\caption{Accuracy for PINN problem, comparing the effects of geometric refinement of the spline knots for the KAN basis. Comparable MLPs are included. Standard deviations computed across random initializations ($N=5$).}
\label{table:results-pinn}
\vskip 0.15in
\begin{center}
\begin{small}
\begin{sc}
\begin{tabular}{lclcc}
\toprule
Type & Layers &  Fidelity  & \# Param. & Loss: mean (stdev) \\
\midrule
KAN & $[2,20,20,1]$  & coarse     &  3400 & $6.635 \times 10^{-3} \, (4.060 \times 10^{-3})$\\
KAN & $[2,20,20,1]$  & fine       &  9700 & $4.072 \times 10^{-3} \,(2.514 \times 10^{-3})$\\
KAN & $[2,20,20,1]$  & multilevel &  9700 & $\mathbf{2.402\times 10^{-5}} \,(1.897\times10^{-5})$ \\
\midrule
MLP & $[2,20,20,1]$  & & 500 & $ 2.334 \times 10^{-2} \, (1.740\times 10^{-3})$ \\
MLP & $[2,56,56,1]$  & & 3416 & $ 1.734\times 10^{-2} (7.989\times 10^{-3})$\\
\bottomrule
\end{tabular}
\end{sc}
\end{small}
\end{center}
\end{table}

\subsubsection{Allen-Cahn Equation} \label{sec:results:pinns:ac}

Last, we consider a PINN problem of learning the solution to the Allen-Cahn equation,
\begin{align*}
    \partial_t u(x,t) =& -\epsilon \partial_{xx} u(x,t) + 5 u(x,t)^3 - 5 u(x,t)\quad& \text{ for } (x,t) &\in [-1,1] \times [0,1]    \\
    u(x,0) =& x^2 \cos(\pi x)\quad& \text{ for } (x,t) &\in [-1,1] \times \{0\}.
\end{align*} 
We use this to again demonstrate the power of multilevel training, as well as the spectral shaping in model output that results from the multilevel training. 

This equation has stable attractors as time increases at $u = \pm 1$ and an unstable fixed point at $u = 0$ that, while minimizing the residual of the PDE, yields a solution that is incompatible with the initial condition. As in the previous examples, we use residual-based attention during training, but otherwise do not use any of the specialized modification and training techniques that are prevalent in state-of-the-art solutions, including (but not limited to) time-causality, more advanced optimizers, and adaptive sampling of collocation points. The point here is to emphasize that the multilevel training of KANs for PINNs is fast and robust {out of the box}, without any modifications. We train each model with Adam optimizer with a learning rate of $0.001$ for $2500$ steps at each resolution, using a sigmoid normalization layer for a $[2,5,5,1]$ architecture using cubic b-splines starting at a resolution of three interior spline knots equispaced on the interval $[-1,1]$. The residual during training is evaluated for the loss at $20,000$ collocation points randomly chosen in the domain.%

\begin{figure}[!htb]
    \centering
    \includegraphics[width=0.95\textwidth]{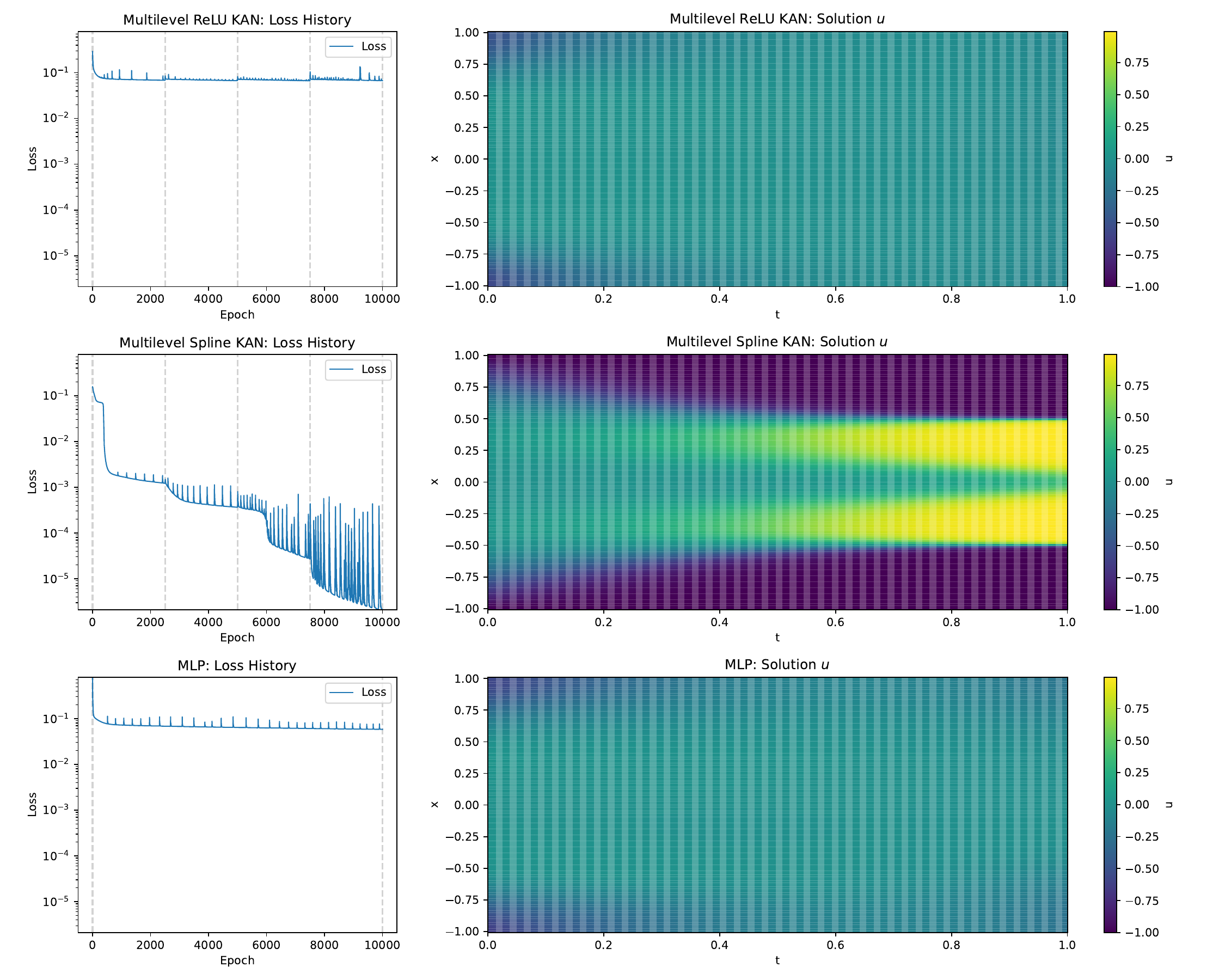}
    \caption{Training histories for multilevel refinement of equivalent KAN architectures in both ReLU and B-spline bases, with a similarly-sized MLP for comparison. \label{fig:pinss:allencahn:history}}
\end{figure}
\begin{figure}[!htb]
    \centering
    \includegraphics[width=0.95\textwidth]{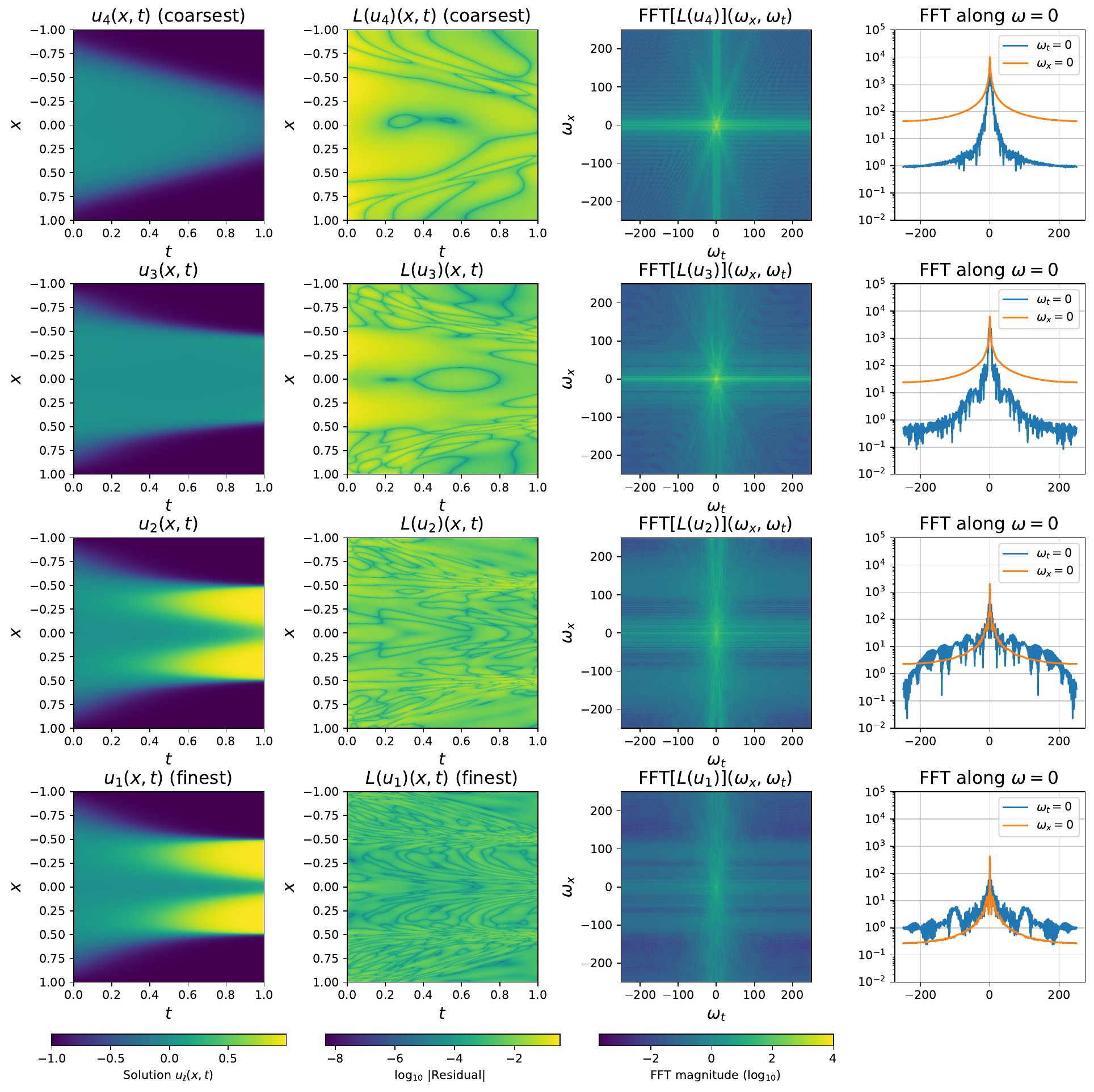}
    \caption{Plots of the trained B-spline-basis solution (leftmost) for the Allen-Cahn Equation in \Cref{sec:results:pinns:ac}, residual in log scale (left center), the Fourier transform of residual (right center), and cross-sections of the Fourier transform along the axes $\omega_t = 0$ and $\omega_x=0$ (rightmost). Each row displays the output after training at each level of refinement, from coarsest (top) to finest (bottom). The first three columns share the same colorbar shown at the bottom of the column.  
    \label{fig:pinns:allencahn:spline}}
\end{figure}

The training loss histories for both bases, and a similarly-sized MLP, are shown in \Cref{fig:pinss:allencahn:history}, along with the obtained physical solution. As previously, we see substantial benefits from multilevel training in the spline basis, including immediate acceleration in training and reduction in loss at each refinement level. In contrast, the ReLU basis is unable to exploit the increased network capacity from refinement, resulting in a very poor solution. Moreover, a conventional MLP of similar size performs similarly to the ReLU KAN, failing to capture any meaningful structure in the solution.

We further examine the residuals and spectra that arise in multilevel training of the spline-based KANs in \Cref{fig:pinns:allencahn:spline} (similar plots for ReLU-based KANs can be found in Supplementary Material). Of note, we clearly see that when using the spline basis, refinement provides not only better accuracy, but also substantial differences in the spectrum of the Fourier modes in the residual: the energy in the Fourier modes decrease with refinement, and the support of the Fourier modes active in the residual widen as we refine the spline representations. These results reinforce the conclusions of our theory from the previous sections, that training hierarchically in the spline basis enables targeted improvements of higher Fourier modes with refinement. In contrast, the ReLU networks exhibit some improvement with increasing spline resolution, but the spectra remains narrow -- as spectral bias implies that multilevel training does not happen in a geometric space in which we can exploit multigrid-style refinement -- and as a result, the accuracy of the model suffers, even though both architectures have the same approximation capabilities and belong to the same class.

\section{Conclusion}

We have presented a theoretical and algorithmic framework for understanding and efficiently training spline-based Kolmogorov-Arnold Networks using multilevel principles. Our main contributions address both fundamental properties of KANs and practical training challenges. From a theoretical perspective, we established that KANs with spline basis functions of order $r$ are equivalent to multichannel MLPs with power ReLU activations through a specific linear change of basis, where the transformation matrix $A^{[r]}$ corresponds to a finite-difference discretization of the $r$th derivative. This equivalence yields immediate algorithmic benefits in a non-recursive implementation of spline-based KANs that is faster than the standard Cox-de Boor implementation. More importantly, our analysis reveals fundamental differences in how gradient-based optimization behaves under different basis representations, despite their equivalence as forward operators. The change of basis matrix $A^{[r]}$ acts as a preconditioner that dramatically affects training dynamics. Specifically, gradient descent in the multichannel MLP (ReLU) basis heavily prioritizes smooth functions over oscillatory ones by orders-of-magnitude scaling. In contrast, the natural spline basis enables strong feature localization through compact support, allowing the network to efficiently learn functions with sharp gradients and low regularity. This theoretical understanding provides rigorous justification for KANs' observed advantages in capturing complex, nonsmooth functions.

Building on these insights, we developed an efficient multilevel training framework based on uniform knot refinement. We introduced the concept of properly nested hierarchies for multilevel optimization, ensuring that a coarse model preserves the functional approximation achieved at finer levels—a critical property that prevents interpolated models from undoing progress. When combined with gradient-based descent methods in the natural spline basis, which automatically provides complementary relaxation across hierarchy levels, this framework achieves remarkable improvements: our numerical experiments consistently demonstrate 2-3 orders of magnitude better accuracy achieved compared to standard training methods or comparable MLPs, while maintaining computational efficiency. In contrast, we also showed that attempting multilevel training with the ReLU basis formulation provides effectively zero improvement over just using the coarse model, confirming that successful multilevel methods require both proper nesting \emph{and} complementary optimization dynamics across levels. This principle has been elusive in multilevel machine learning but is naturally satisfied by KANs in their native spline representation.

More broadly, our multilevel framework offers a principled demonstration that, with appropriate structure and complementary dynamics across levels, multigrid ideas can indeed provide the efficiency gains in neural network training that they have long delivered in scientific computing. Future work will consider multilevel cycling, and extensions to other architectures and types of coarsening and refinement.

\appendix

\section{Notational Shift for MLP Architectures} 
\label{app:notational_shift}

Conventionally, an MLP with $L$ hidden layers, starting with an input $x = x^{(0)}$, has each layer $\ell=0,\dots,L-1$ expressed as
$
x^{(\ell+1)} = \sigma\left( W^{(\ell)} x^{(\ell)} + b^{(\ell)} \right),
$
with the final output  
$NN(x) = x^{(L+1)} = W^{(L)} x^{(\ell)}.$
This compositional structure can be broken down more granularly, as
\begin{equation} 
x^{(\ell + \frac{1}{2})} = W^{(\ell)} x^{(\ell)}, \hspace{4ex}
x^{(\ell + 1)} = \sigma \left( x^{(\ell + \frac{1}{2})} + b^{(\ell)} \right),  
\label{eq:granular} \end{equation}
where then the final layer omits the nonlinear activation step so that $ NN(x) = x^{(L + \frac{1}{2})}$ instead of $NN(x) = x^{(L+1)}.$

The rearranged layers used in our analysis, which post-multiply layers by learnable weights instead of pre-multiplying by learnable weights, simply combine the expressions in \eqref{eq:granular} for $\ell = 1, \dots, L-1$ into a single step,
$
x^{(\ell+\frac{1}{2})} = W^{(\ell)} \sigma\left( x^{(\ell - \frac{1}{2})} + b^{(\ell-1)} \right),
$
yielding the same computation. Notationally, we substitute $t^{(\ell)} = - b^{(\ell)}$, to match the conventions from spline literature. This yields our ultimate expression for each layer, 
\begin{equation}
x^{(\ell+\frac12)} = W^{(\ell)} \sigma\left(x^{(\ell-\frac12)}-t^{(\ell-1)}\right).
\end{equation}

\section{Proofs regarding change of basis matrix $A^{[r]}$}
\label{app:cob_proof}

First, we prove the following result regarding ReLU functions.  
\begin{proposition}\label{lemma:relu_power_recursion}
    For any scalars $a,b$, and for $r > 1$,
    \begin{equation}
        (x - a) \text{ReLU}(x - b)^r = \text{ReLU}(x-b)^{r+1} + (b-a) \text{ReLU}(x - b)^r.
    \end{equation}
\end{proposition} 
\begin{proof}
    If $x \le b$, then both sides equal zero. Thus, it suffices to consider $x > b$, where then
    $\text{ReLU}(x-b) = x-b.$
    Thus, it suffices to show that 
    $(x-a)(x-b)^r = (x-b)^{r+1} + (b-a)(x-b)^r.$
    Indeed, factoring the righthand side, we see
    \begin{equation} \begin{split}
        (x-b)^{r+1} + (b-a)(x-b)^r &= (x-b)^r \left( (x-b) + (b-a) \right) = (x-b)^r (x-a),
    \end{split} \end{equation}
    which is the desired result.
\end{proof}

We now derive a recursive definition for $A^{[r]}$, as stated in \eqref{eq:A_recurrence}.
\begin{lemma}[\Cref{lemma:basis}]
    For $r > 1$, let $A^{[r-1]}$ denote splines of order $r-1$ constructed on knots for splines of order $r$. Then the matrix $A^{[r]}$ is defined entry-wise as:
    \begin{equation} \label{eq:A_recurrence2}
         A^{[r]}_{ij} = \frac{1}{t_{i+r-1} - t_i} A^{[r-1]}_{ij} - \frac{1}{t_{i+r} - t_{i+1}} A^{[r-1]}_{i+1,j},
    \end{equation}
    where we define $A^{[r-1]}_{n+r,:} = 0$ to account for the special case of the final row $i=n+r-1$. 
\end{lemma}
\begin{proof}
It is easily verified that 
\begin{equation} \label{eq:A_1}
A^{[1]}_{ij} = \begin{cases} 1 & j = i \\ -1 & j = i+1 \end{cases}.
\end{equation}
We now prove the main statement by induction. For the base case of $r=2$, we have the known (e.g. \cite{hong2022activation}) result that
\begin{equation} \label{eq:A_2}
    A^{[2]}_{ij} = \begin{cases} \frac{1}{t_{i+1} - t_i} & j = i \\
    -\left(\frac{1}{t_{i+1} - t_i} + \frac{1}{t_{i+2}-t_{i+1}} \right) & j = i+1 \\
    \frac{1}{t_{i+2}-t_{i+1}} & j = i+2 \\
    0 & \text{else}
    \end{cases}.
\end{equation}
Directly from \eqref{eq:A_1} and \eqref{eq:A_2}, we see that $A^{[2]}$ satisfies Equation \eqref{eq:A_recurrence2} for $r=2$.
We proceed now to the inductive step for $r \ge 2$. Suppose the result holds through $s=1,\dots,r-1$. By the recurrence relation used to define $b^{[r]}_i$, we have
\begin{align*}
    b^{[r]}_i(x) &= \frac{x - t_i}{t_{i+r-1} - t_i} b^{[r-1]}_i(x) + \frac{t_{i+r} - x}{t_{i+r} - t_{i+1}} b^{[r-1]}_{i+1}(x) \\
    &= \left( \frac{1}{t_{i+r-1} - t_i} \right) \sum_{j} A^{[r-1]}_{ij} \left( x- t_i \right)  \text{ReLU}(x - t_j)^{r-2} - \\
    &\qquad\qquad \left(\frac{1}{t_{i+r} - t_{i+1}} \right) \sum_{j} A^{[r-1]}_{i+1,j} \left( x - t_{i+r} \right) \text{ReLU}(x - t_j)^{r-2}.
\end{align*}
By Lemma \ref{lemma:relu_power_recursion}, 
\begin{align*}
    b^{[r]}_i(x) 
    &= \left( \frac{1}{t_{i+r-1} - t_i} \right) \sum_{j} A^{[r-1]}_{ij} \left( x- t_i \right)  \text{ReLU}(x - t_j)^{r-2} - \\
    &\qquad\qquad \left(\frac{1}{t_{i+r} - t_{i+1}} \right) \sum_{j} A^{[r-1]}_{i+1,j} \left( x - t_{i+r} \right) \text{ReLU}(x - t_j)^{r-2} \\
    &= \left( \frac{1}{t_{i+r-1} - t_i} \right) \sum_{j} A^{[r-1]}_{ij} \left( \text{ReLU}(x - t_j)^{r-1} + (t_j - t_i) \text{ReLU}(x - t_j)^{r-2} \right) - \\
    & \left(\frac{1}{t_{i+r} - t_{i+1}} \right) \sum_{j} A^{[r-1]}_{i+1,j} \left( \text{ReLU}(x - t_j)^{r-1} + (t_j - t_{i+r}) \text{ReLU}(x - t_j)^{r-2} \right).
\end{align*}
Since $b_i^{[r]} \in C^{r-2}([a,b])$ and $\text{ReLU}(\, \cdot - t_j)^{r-1} \in C^{r-2}([a,b])$ but $\text{ReLU}(\, \cdot - t_j)^{r-2} \notin C^{r-2}([a,b])$,
the coefficients of the $\text{ReLU}(\cdot - t_j)^{r-2}$ terms must cancel. This is specifically because expanding the $\text{ReLU}(\, \cdot - t_j)^{r-2}$ terms as piecewise polynomials, none of the expanded terms are sufficiently differentiable at the knot, and thus all terms must cancel for the full summation to be in $C^{r-2}([a,b])$. This can also be shown directly through laborious arithmetic calculation. Therefore,
\begin{align*}
    b^{[r]}_i(x) 
    &= \left( \frac{1}{t_{i+r-1} - t_i} \right) \sum_{j} A^{[r-1]}_{ij} \text{ReLU}(x - t_j)^{r-1}  - \\&\quad\quad \left(\frac{1}{t_{i+r} - t_{i+1}} \right) \sum_{j} A^{[r-1]}_{i+1,j} \text{ReLU}(x - t_j)^{r-1} \\
    &= \sum_{j} \left( \frac{1}{t_{i+r-1} - t_i}   A^{[r-1]}_{ij}  - \frac{1}{t_{i+r} - t_{i+1}}  A^{[r-1]}_{i+1,j} \right) \text{ReLU}(x - t_j)^{r-1},
\end{align*}
which completes the proof.
\end{proof}

We now derive a closed-form for $A^{[r]}$ for uniform knots. For $r=2$, this is the matrix in \cite{hong2022activation}, $A = I - 2S + S^2$, for shift operator $S\mathbf{e}_j = \mathbf{e}_{j+1}$.
\begin{corollary*}[\Cref{cor:uni}]
In the case of uniform knots, with spacing $h = t_i - t_{i-1} = \frac{1}{n}$, $A^{[r]}$ takes the direct form
\begin{equation}\label{eq:uni-Ar}
A^{[r]} = \frac{h^{1-r}}{(r-1)!} (A^{[1]})^r.
\end{equation}
Moreover, $A^{[r]}$ is upper triangular Toeplitz, with 
entries
    \begin{equation}\label{eq:uni-Ar_ij}
        A^{[r]}_{ij} = \begin{cases}
            \frac{(-1)^{j-i} }{(j-i)!\,(r-j+i)!} \, \frac{r}{h^{r-1}}  & i \le j \le i+r \\
            0 & \text{else}
        \end{cases}.
    \end{equation}
\end{corollary*}
\begin{proof}
    This follows by noting that if $t_{i+r-1}-t_i = t_{i+r}-t_{i+1} = h(r-1)$, \eqref{eq:A_recurrence} corresponds to a left scaling of $A^{[r-1]}$ by $\tfrac{1}{h(r-1)}A^{[1]}$ \eqref{eq:A_1}. Repeating this for $r=2,3,...$ yields \eqref{eq:uni-Ar}. For the Toeplitz property, from \eqref{eq:uni-Ar} we see that $A^{[r]}$ is a power of $A^{[1]}$. Since $A^{[1]}$ is an upper bidiagonal Toeplitz matrix, it follows that $A^{[r]}$ is also upper triangular Toeplitz (see, e.g. \cite[Lemma 22]{southworth2019necessary}), with coefficients given in \eqref{eq:uni-Ar_ij}.
\end{proof}

\subsection*{Acknowledgements}
This work was funded in part by the National Nuclear Security Administration Interlab Laboratory Directed Research and Development program under project number 20250861ER. Los Alamos National Laboratory report number LA-UR-26-21552. 
This article has been authored by employees of National Technology \& Engineering Solutions of Sandia, LLC under Contract No. DE-NA0003525 with the U.S. Department of Energy (DOE). The employees own all right, title and interest in and to the article and are solely responsible for its contents. The United States Government retains and the publisher, by accepting the article for publication, acknowledges that the United States Government retains a non-exclusive, paid-up, irrevocable, world-wide license to publish or reproduce the published form of this article or allow others to do so, for United States Government purposes. The DOE will provide public access to these results of federally sponsored research in accordance with the DOE Public Access Plan \url{https://www.energy.gov/downloads/doe-public-access-plan}
BSS would like to thank Achi Brandt and Oren Livne for many early and engaging discussions on multilevel methods for machine learning. 

\bibliographystyle{siamplain}
\bibliography{refs}

\begin{thebibliography}{10}

\bibitem{abueidda2024deepokan}
{\sc D.~W. Abueidda, P.~Pantidis, and M.~E. Mobasher}, {\em Deepokan: Deep
  operator network based on kolmogorov arnold networks for mechanics problems},
  arXiv preprint arXiv:2405.19143,  (2024).

\bibitem{actor2018computation}
{\sc J.~Actor}, {\em Computation for the kolmogorov superposition theorem},
  master's thesis, Rice University, 2018.

\bibitem{anagnostopoulos2024}
{\sc S.~J. Anagnostopoulos, J.~D. Toscano, N.~Stergiopulos, and G.~E.
  Karniadakis}, {\em Residual-based attention in physics-informed neural
  networks}, Computer Methods in Applied Mechanics and Engineering, 421 (2024),
  p.~116805, \url{https://doi.org/https://doi.org/10.1016/j.cma.2024.116805},
  \url{https://www.sciencedirect.com/science/article/pii/S0045782524000616}.

\bibitem{Botsaris.1978}
{\sc C.~Botsaris}, {\em Differential gradient methods}, Journal of Mathematical
  Analysis and Applications, 63 (1978), pp.~177--198,
  \url{https://doi.org/10.1016/0022-247x(78)90114-2}.

\bibitem{braun2009constructive}
{\sc J.~Braun and M.~Griebel}, {\em On a constructive proof of kolmogorov’s
  superposition theorem}, Constructive approximation, 30 (2009), pp.~653--675.

\bibitem{bresson2024kagnns}
{\sc R.~Bresson, G.~Nikolentzos, G.~Panagopoulos, M.~Chatzianastasis, J.~Pang,
  and M.~Vazirgiannis}, {\em Kagnns: Kolmogorov-arnold networks meet graph
  learning}, arXiv preprint arXiv:2406.18380,  (2024).

\bibitem{cacciatore2024preliminary}
{\sc A.~Cacciatore, V.~Morelli, F.~Paganica, E.~Frontoni, L.~Migliorelli, and
  D.~Berardini}, {\em A preliminary study on continual learning in computer
  vision using kolmogorov-arnold networks}, arXiv preprint arXiv:2409.13550,
  (2024).

\bibitem{cang2024can}
{\sc Y.~Cang, L.~Shi, et~al.}, {\em Can kan work? exploring the potential of
  kolmogorov-arnold networks in computer vision}, arXiv preprint
  arXiv:2411.06727,  (2024).

\bibitem{Chang2018}
{\sc B.~Chang, L.~Meng, E.~Haber, F.~Tung, and D.~Begert}, {\em Multi-level
  residual networks from dynamical systems view}, in International Conference
  on Learning Representations, 2018.

\bibitem{chen2018neural}
{\sc R.~T. Chen, Y.~Rubanova, J.~Bettencourt, and D.~K. Duvenaud}, {\em Neural
  ordinary differential equations}, Advances in neural information processing
  systems, 31 (2018).

\bibitem{cheon2024demonstrating}
{\sc M.~Cheon}, {\em Demonstrating the efficacy of kolmogorov-arnold networks
  in vision tasks}, arXiv preprint arXiv:2406.14916,  (2024).

\bibitem{chui1988multivariate}
{\sc C.~K. Chui}, {\em Multivariate splines}, SIAM, 1988.

\bibitem{cyr2025torchbraid}
{\sc E.~C. Cyr, J.~Hahne, N.~S. Moore, J.~B. Schroder, B.~S. Southworth, and
  D.~A. Vargas}, {\em Torchbraid: High-performance layer-parallel training of
  deep neural networks with mpi and gpu acceleration}, ACM Transactions on
  Mathematical Software, 51 (2025), pp.~1--30.

\bibitem{deboor1978practical}
{\sc C.~De~Boor}, {\em A practical guide to splines}, vol.~27, springer New
  York, 1978.

\bibitem{dong2024kolmogorov}
{\sc C.~Dong, L.~Zheng, and W.~Chen}, {\em Kolmogorov-arnold networks (kan) for
  time series classification and robust analysis}, arXiv preprint
  arXiv:2408.07314,  (2024).

\bibitem{Eliasof2023}
{\sc M.~Eliasof, J.~Ephrath, L.~Ruthotto, and E.~Treister}, {\em {MGIC}:
  Multigrid-in-channels neural network architectures}, {SIAM} Journal on
  Scientific Computing, 45 (2023), pp.~S307--S328,
  \url{https://doi.org/10.1137/21m1430194}.

\bibitem{Feischl2024t1}
{\sc M.~Feischl, A.~Rieder, and F.~Zehetgruber}, {\em Towards optimal
  hierarchical training of neural networks}, {arXiv},  (2024),
  \url{https://doi.org/10.48550/arxiv.2407.02242},
  \url{https://arxiv.org/abs/2407.02242}.

\bibitem{GaedkeMerzhauser2021}
{\sc L.~Gaedke-Merzh\"auser, A.~Kopaničáková, and R.~Krause}, {\em
  Multilevel minimization for deep residual networks}, {ESAIM}: Proceedings and
  Surveys, 71 (2021), pp.~131--144,
  \url{https://doi.org/10.1051/proc/202171131}.

\bibitem{gao2024convergence}
{\sc Y.~Gao and V.~Y. Tan}, {\em On the convergence of (stochastic) gradient
  descent for kolmogorov--arnold networks}, arXiv preprint arXiv:2410.08041,
  (2024).

\bibitem{Gunther2020}
{\sc S.~G\"unther, L.~Ruthotto, J.~B. Schroder, E.~C. Cyr, and N.~R. Gauger},
  {\em Layer-parallel training of deep residual neural networks}, {SIAM}
  Journal on Mathematics of Data Science, 2 (2020), pp.~1--23,
  \url{https://doi.org/10.1137/19m1247620}.

\bibitem{haber2017stable}
{\sc E.~Haber and L.~Ruthotto}, {\em Stable architectures for deep neural
  networks}, Inverse problems, 34 (2017), p.~014004.

\bibitem{hackbusch2013multi}
{\sc W.~Hackbusch}, {\em Multi-grid methods and applications}, vol.~4, Springer
  Science \& Business Media, 2013.

\bibitem{He2019}
{\sc J.~He and J.~Xu}, {\em {MgNet}: A unified framework of multigrid and
  convolutional neural network}, Science China Mathematics, 62 (2019),
  pp.~1331--1354, \url{https://doi.org/10.1007/s11425-019-9547-2},
  \url{https://arxiv.org/abs/1901.10415}.

\bibitem{he2024mlp}
{\sc Y.~He, Y.~Xie, Z.~Yuan, and L.~Sun}, {\em Mlp-kan: Unifying deep
  representation and function learning}, arXiv preprint arXiv:2410.03027,
  (2024).

\bibitem{hollig2003finite}
{\sc K.~H{\"o}llig}, {\em Finite element methods with B-splines}, SIAM, 2003.

\bibitem{hong2022activation}
{\sc Q.~Hong, J.~W. Siegel, Q.~Tan, and J.~Xu}, {\em On the activation function
  dependence of the spectral bias of neural networks}, arXiv preprint
  arXiv:2208.04924,  (2022).

\bibitem{hou2024comprehensive}
{\sc Y.~Hou and D.~Zhang}, {\em A comprehensive survey on kolmogorov arnold
  networks (kan)}, arXiv preprint arXiv:2407.11075,  (2024).

\bibitem{howard2024finite}
{\sc A.~A. Howard, B.~Jacob, S.~H. Murphy, A.~Heinlein, and P.~Stinis}, {\em
  Finite basis kolmogorov-arnold networks: domain decomposition for data-driven
  and physics-informed problems}, arXiv preprint arXiv:2406.19662,  (2024).

\bibitem{igelnik2003kolmogorov}
{\sc B.~Igelnik and N.~Parikh}, {\em Kolmogorov's spline network}, IEEE
  transactions on neural networks, 14 (2003), pp.~725--733.

\bibitem{jacob2024spikans}
{\sc B.~Jacob, A.~A. Howard, and P.~Stinis}, {\em Spikans: Separable
  physics-informed kolmogorov-arnold networks}, arXiv preprint
  arXiv:2411.06286,  (2024).

\bibitem{jacot2018neural}
{\sc A.~Jacot, F.~Gabriel, and C.~Hongler}, {\em Neural tangent kernel:
  Convergence and generalization in neural networks}, Advances in neural
  information processing systems, 31 (2018).

\bibitem{ke2017multigrid}
{\sc T.-W. Ke, M.~Maire, and S.~X. Yu}, {\em Multigrid neural architectures},
  in Proceedings of the IEEE Conference on Computer Vision and Pattern
  Recognition (CVPR), 2017, pp.~6665--6673.

\bibitem{kiamari2024gkan}
{\sc M.~Kiamari, M.~Kiamari, and B.~Krishnamachari}, {\em Gkan: Graph
  kolmogorov-arnold networks}, arXiv preprint arXiv:2406.06470,  (2024).

\bibitem{kingma2013auto}
{\sc D.~P. Kingma, M.~Welling, et~al.}, {\em Auto-encoding variational bayes},
  2013.

\bibitem{koenig2024kan}
{\sc B.~C. Koenig, S.~Kim, and S.~Deng}, {\em Kan-odes: Kolmogorov--arnold
  network ordinary differential equations for learning dynamical systems and
  hidden physics}, Computer Methods in Applied Mechanics and Engineering, 432
  (2024), p.~117397.

\bibitem{kolda2009tensor}
{\sc T.~G. Kolda and B.~W. Bader}, {\em Tensor decompositions and
  applications}, SIAM review, 51 (2009), pp.~455--500.

\bibitem{kolmogorov1957representation}
{\sc A.~N. Kolmogorov}, {\em On the representation of continuous functions of
  many variables by superposition of continuous functions of one variable and
  addition}, in Doklady Akademii Nauk, vol.~114:5, Russian Academy of Sciences,
  1957, pp.~953--956.

\bibitem{Kopanicakova2022}
{\sc A.~Kopaničáková and R.~Krause}, {\em Globally convergent multilevel
  training of deep residual networks}, {SIAM} Journal on Scientific Computing,
  0 (2022), pp.~S254--S280, \url{https://doi.org/10.1137/21m1434076}.

\bibitem{leni2013kolmogorov}
{\sc P.-E. Leni, Y.~D. Fougerolle, and F.~Truchetet}, {\em The kolmogorov
  spline network for image processing}, in Image Processing: Concepts,
  Methodologies, Tools, and Applications, IGI Global, 2013, pp.~54--78.

\bibitem{liu2024kan}
{\sc Z.~Liu, Y.~Wang, S.~Vaidya, F.~Ruehle, J.~Halverson,
  M.~Solja{\v{c}}i{\'c}, T.~Y. Hou, and M.~Tegmark}, {\em Kan:
  Kolmogorov-arnold networks}, arXiv preprint arXiv:2404.19756,  (2024).

\bibitem{lorentz1962metric}
{\sc G.~Lorentz}, {\em Metric entropy, widths, and superpositions of
  functions}, The American Mathematical Monthly, 69 (1962), pp.~469--485.

\bibitem{mcculloch1943logical}
{\sc W.~S. McCulloch and W.~Pitts}, {\em A logical calculus of the ideas
  immanent in nervous activity}, The bulletin of mathematical biophysics, 5
  (1943), pp.~115--133.

\bibitem{noorizadegan2025practitioner}
{\sc A.~Noorizadegan, S.~Wang, and L.~Ling}, {\em A practitioner's guide to
  kolmogorov-arnold networks}, arXiv preprint arXiv:2510.25781,  (2025).

\bibitem{patra2024physics}
{\sc S.~Patra, S.~Panda, B.~K. Parida, M.~Arya, K.~Jacobs, D.~I. Bondar, and
  A.~Sen}, {\em Physics informed kolmogorov-arnold neural networks for
  dynamical analysis via efficent-kan and wav-kan}, arXiv preprint
  arXiv:2407.18373,  (2024).

\bibitem{pinkus1999approximation}
{\sc A.~Pinkus}, {\em Approximation theory of the mlp model in neural
  networks}, Acta numerica, 8 (1999), pp.~143--195.

\bibitem{qiu2024relu}
{\sc Q.~Qiu, T.~Zhu, H.~Gong, L.~Chen, and H.~Ning}, {\em Relu-kan: New
  kolmogorov-arnold networks that only need matrix addition, dot
  multiplication, and relu}, arXiv preprint arXiv:2406.02075,  (2024).

\bibitem{qiu2025powermlp}
{\sc R.~Qiu, Y.~Miao, S.~Wang, Y.~Zhu, L.~Yu, and X.-S. Gao}, {\em Powermlp: An
  efficient version of kan}, in Proceedings of the AAAI Conference on
  Artificial Intelligence, vol.~39:19, 2025, pp.~20069--20076.

\bibitem{rigas2024adaptive}
{\sc S.~Rigas, M.~Papachristou, T.~Papadopoulos, F.~Anagnostopoulos, and
  G.~Alexandridis}, {\em Adaptive training of grid-dependent physics-informed
  kolmogorov-arnold networks}, IEEE Access,  (2024).

\bibitem{rosenblatt1958perceptron}
{\sc F.~Rosenblatt}, {\em The perceptron: a probabilistic model for information
  storage and organization in the brain.}, Psychological review, 65 (1958),
  p.~386.

\bibitem{ruthotto2020deep}
{\sc L.~Ruthotto and E.~Haber}, {\em Deep neural networks motivated by partial
  differential equations}, Journal of Mathematical Imaging and Vision, 62
  (2020), pp.~352--364.

\bibitem{shukla2024comprehensive}
{\sc K.~Shukla, J.~D. Toscano, Z.~Wang, Z.~Zou, and G.~E. Karniadakis}, {\em A
  comprehensive and fair comparison between mlp and kan representations for
  differential equations and operator networks}, arXiv preprint
  arXiv:2406.02917,  (2024).

\bibitem{sloane1995encyclopedia}
{\sc N.~J.~A. Sloane and S.~Plouffe}, {\em The encyclopedia of integer
  sequences}, (No Title),  (1995).

\bibitem{so2024higher}
{\sc C.~C. So and S.~P. Yung}, {\em Higher-order-relu-kans (hrkans) for solving
  physics-informed neural networks (pinns) more accurately, robustly and
  faster}, arXiv preprint arXiv:2409.14248,  (2024).

\bibitem{somvanshi2024survey}
{\sc S.~Somvanshi, S.~A. Javed, M.~M. Islam, D.~Pandit, and S.~Das}, {\em A
  survey on kolmogorov-arnold network}, arXiv preprint arXiv:2411.06078,
  (2024).

\bibitem{southworth2019necessary}
{\sc B.~S. Southworth}, {\em Necessary conditions and tight two-level
  convergence bounds for parareal and multigrid reduction in time}, SIAM
  Journal on Matrix Analysis and Applications, 40 (2019), pp.~564--608.

\bibitem{sprecher1993universal}
{\sc D.~A. Sprecher}, {\em A universal mapping for kolmogorov's superposition
  theorem}, Neural networks, 6 (1993), pp.~1089--1094.

\bibitem{sprecher2017algebra}
{\sc D.~A. Sprecher}, {\em From Algebra to Computational Algorithms: Kolmogorov
  and Hilbert's Problem 13}, Docent Press, 2017.

\bibitem{tilli1998singular}
{\sc P.~Tilli}, {\em Singular values and eigenvalues of non-hermitian block
  toeplitz matrices}, Linear algebra and its applications, 272 (1998),
  pp.~59--89.

\bibitem{toscano2024pinns}
{\sc J.~D. Toscano, V.~Oommen, A.~J. Varghese, Z.~Zou, N.~A. Daryakenari,
  C.~Wu, and G.~E. Karniadakis}, {\em From pinns to pikans: Recent advances in
  physics-informed machine learning}, arXiv preprint arXiv:2410.13228,  (2024).

\bibitem{tseng2023cusp}
{\sc Y.-H. Tseng, T.-S. Lin, W.-F. Hu, and M.-C. Lai}, {\em A cusp-capturing
  pinn for elliptic interface problems}, Journal of Computational Physics, 491
  (2023), p.~112359.

\bibitem{vaca2024kolmogorov}
{\sc C.~J. Vaca-Rubio, L.~Blanco, R.~Pereira, and M.~Caus}, {\em
  Kolmogorov-arnold networks (kans) for time series analysis}, arXiv preprint
  arXiv:2405.08790,  (2024).

\bibitem{vaswani2017attention}
{\sc A.~Vaswani}, {\em Attention is all you need}, Advances in Neural
  Information Processing Systems,  (2017).

\bibitem{wang2025simulating}
{\sc S.~Wang, S.~Sankaran, X.~Fan, P.~Stinis, and P.~Perdikaris}, {\em
  Simulating three-dimensional turbulence with physics-informed neural
  networks}, arXiv preprint arXiv:2507.08972,  (2025).

\bibitem{wang2023expert}
{\sc S.~Wang, S.~Sankaran, H.~Wang, and P.~Perdikaris}, {\em An expert's guide
  to training physics-informed neural networks}, arXiv preprint
  arXiv:2308.08468,  (2023).

\bibitem{wu2024kolmogorov}
{\sc Y.~Wu, T.~Su, B.~Du, S.~Hu, J.~Xiong, and D.~Pan}, {\em Kolmogorov--arnold
  network made learning physics laws simple}, The Journal of Physical Chemistry
  Letters, 15 (2024), pp.~12393--12400.

\bibitem{yu2024kan}
{\sc R.~Yu, W.~Yu, and X.~Wang}, {\em Kan or mlp: A fairer comparison}, arXiv
  preprint arXiv:2407.16674,  (2024).

\bibitem{zeng2024kan}
{\sc C.~Zeng, J.~Wang, H.~Shen, and Q.~Wang}, {\em Kan versus mlp on irregular
  or noisy functions}, arXiv preprint arXiv:2408.07906,  (2024).

\end{thebibliography}

\newpage
\begin{center}
\LARGE Supplementary Material
\end{center}

\section{Neural tangent kernel and change of basis}

The Neural Tangent Kernel (NTK) is a powerful tool for understanding the training dynamics of neural networks in the infinite-width limit \cite{jacot2018neural}. The NTK describes how the network's output changes with respect to its parameters during training. For a neural network $f(\bx;\bw)$ with parameters $\bw$, the NTK is defined as $\mathbf{W}(\bx, \bx') \coloneqq \left\langle \frac{\partial f(\bx;\bw)}{\partial \bw}, \frac{\partial f(\bx';\bw)}{\partial \bw} \right\rangle$. For a batch of $k$ data points, we can organize these gradients into a Jacobian matrix $J \in \mathbb{R}^{k \times N_w}$, where each row contains $\frac{\partial f(\bx^{(i)};\bw)}{\partial \bw}$ for the $i$-th data point. The NTK matrix for this batch is then given by the Gram matrix $\text{NTK} = JJ^T \in \mathbb{R}^{k \times k}$, whose $(i,j)$-th entry is precisely $\mathbf{W}(\bx^{(i)}, \bx^{(j)})$. Under our change of basis from spline weights $\bu$ to ReLU weights $\bw = A^T\bu$, the Jacobian transforms as $J_g = J_f A^T$, yielding $\text{NTK}_S^{[r]} = J_f A^{[r]}(A^{[r]})^T J_f^T$ for the spline basis.

The following theorem introduces a specific result from the field of block-Toeplitz operator theory, which is used in proving the theorem that follows. 
\begin{theorem}[Maximum singular value of block-Toeplitz operators \cite{tilli1998singular}]\label{th:toeplitz2}
Let $T_N(F)$ be an $N\times N$ block-Toeplitz matrix, with continuous generating function $F(x): [0,2\pi]\to\mathbb{C}^{m\times m}$.
Then, the maximum singular value is bounded above by
\begin{align*}
    \sigma_{max}(T_N(F)) \leq \max_{x\in[0,2\pi]} \sigma_{max}(F(x)),
\end{align*}
 for all $N\in\mathbb{N}$.
\end{theorem}

We now relate the spectral radius of the NTK matrix under a ReLU and spline basis. 

\begin{theorem}
For symmetric positive semidefinite $M$ let $\rho(M)\geq 0$ denote the spectral radius. 
Then for uniform knot spacing,   
\begin{equation}
    \rho(\text{NTK}_S^{[r]}) \leq 4\rho(\text{NTK}_R).
\end{equation}
\end{theorem}
\begin{proof}
Define an auxiliary equivalent ReLU$^{r-1}$ basis 
$\varphi_i^{[r]}(x) = \text{ReLU}\left( \frac{x - t_i}{h} \right)^{r-1}.$
As ReLU is a homogeneous function with respect to positive scalars, this basis is equivalent again to the ReLU basis.

We therefore repeat the process of the above construction, yielding an equivalent change-of-basis matrix $\widetilde{A}^{[r]}$, where
$\widetilde{A}^{[r]} = \frac{1}{h^{1-r}} A^{[r]}.$
Let $J$ denote the Jacobian of $f$ with respect to the weights $W$, so that
\begin{equation}
\text{NTK}_R = J J^T , \hspace{4ex}
\text{NTK}_S^{[r]} = J \mathbb{A}\mathbb{A}^T J^T,
\end{equation}
where $\mathbb{A}$ is the reordered block diagonal operator of layer-specific $A^{[r]}$.

Note that $JJ^T$ and $J\mathbb{A}\mathbb{A}^TJ^T$ are symmetric with nonnegative eigenvalues, and largest eigenvalue given by the square of the largest singular value of $J$ and $J\mathbb{A}$, respectively. Noting that the largest singular value of an operator is given by the $\ell^2$-norm, by sub-multiplicativity $\|J\mathbb{A}\| \leq \|J\|\|\mathbb{A}\|$. Thus we will bound the spectrum of $JA$ by considering the maximum singular value of the change of basis matrix $\mathbb{A}$. Recall in the correct ordering $\mathbb{A}$ is a block-diagonal matrix, with diagonal blocks given by $\widetilde{A} ^{[r]}$ for each layer. The maximum singular value of a block-diagonal matrix is given by the maximum over maximum singular values of each block. Thus, we will proceed by considering $\|\widetilde{A}^{[r]}\|$ for arbitrary $P,Q$ and fixed $r$.

From Corollary \ref{cor:uni}, for uniform knots $\widetilde{A}^{[r]}$ is upper triangular Toeplitz. Appealing to Toeplitz matrix theory, asymptotically (in $n$) tight bounds can be placed on the maximum singular value of $\widetilde{A}^{[r]}$ (and thus equivalently the maximum eigenvalue of $\widetilde{A}^{[r]}(\widetilde{A}^{[r]})^T$) by way of considering the operator's generator function. 
Let $\alpha_\ell$ denote the Toeplitz coefficient for the $i$th diagonal of a Toeplitz matrix, where $\alpha_0$ is the diagonal, $\alpha_{-1}$ the first subdiagonal, and so on. Then the Toeplitz matrix corresponds with a Fourier generator function,
${F}(x) = \sum_{\ell=-\infty}^\infty \alpha_\ell e^{-\mathrm{i}\ell x}.$
In our case we have scalar (i.e. blocksize $N=1$) generating function of $\widetilde{A}^{[r]}$ given by
\begin{equation}
F_r(x) = r \sum_{\ell=0}^r \frac{(-1)^{\ell} }{(\ell)!\,(r-\ell)!} \, e^{-\mathrm{i}\ell x} = r \cdot \frac{(1-e^{\mathrm{i}x})^r}{r!}
= n^{r-1}\frac{(1-e^{\mathrm{i}x})^r}{(r-1)!}.
\end{equation}
Thus from \Cref{th:toeplitz2}, we have for $r\in\mathbb{Z}^+$
\begin{equation}
    \|\widetilde{A}^{[r]}\| \leq \max_x |F_r(x)| = \frac{2^r}{(r-1)!} \leq 4
    \quad\text{for }r\geq1,
\end{equation}
which completes the proof.
\end{proof}

\section{Hyper-parameters and Specializations}

The examples from this paper were run on Nvidia A100 GPUs with a variety of different problem configurations and seeds.

\subsection{Function Regression}
The function regression example in Section 5.1 utilizes a dataset with 20000 random uniform points on the grid $[0.0001,0.9999]^2$.
The network output is normalized by an affine transformation to $[0,1]$.
The random number generator seeds are fixed at 1234, and any ensemble of five runs utilizes seeds 1234 through 1238. The optimizer is L-BFGS with a learning rate of 1.0.

\subsection{Physics Informed Neural Networks (PINNs)}

The 2d Poisson problem uses residual based attention (RBA)~\cite{anagnostopoulos2024} applied to the volumetric term. Assume their are $N_V = |V|$ points in set $V$, and initialize a weighting vector $\vec{\lambda}$ such that $\lambda_i = 1$ where $i=1\ldots N$. The RBA scheme provides the update rule 
\begin{equation}
\lambda_i = (1-\mu) \lambda_i + \mu \frac{e_i}{\max_j e_j} \text{ where }
e_i = \bigl|\nabla\!\cdot\!(\epsilon(\vec x)\,\nabla u_N(\vec x))
           \!-\! f(\vec x)\bigr|.
\end{equation}
The weighting vector scales each point in the volumetric loss term and is consequently incorporated into the gradient. If $e_i$ is relatively large the contribution of that error to the loss is more heavily weighted. 

The multilevel KANs training algorithm employs a linear ramp of the learning rate at the start of training, and at the transition between levels. In the tables we denote $\eta_0$ as the starting learning rate, and $\eta$ as the final target learning rate. The number of steps (or epochs for single batch PINNs training) that is required to transition between them is denoted $S_{LR}$. 

    {
\vspace{12pt}  
    \addtolength{\tabcolsep}{-0.4em}
    \begin{tabular}{lll}
      \multicolumn{3}{c}{\bf Section 5.2.1: 2d Poisson Hyper-Parameters} \\
      \toprule
      {\bf Parameter\;\;}      & {\bf Symbol} \;\; & {\bf Value} \\
      \midrule
        Loss Scaling & $\Gamma_V$ & $10^{-2}$ \\
                   & $\Gamma_I$ & $10^{-1}$ \\
                   & $\Gamma_B$ & $10^{0}$ \\
      Volume Points & $|V|$ & 2401 \\
      Boundary Points & $|B|$ & 200 \\
      Interface Points & $|I|$ & 49 \\
      \midrule
      MLP Layers & & $[2,16,16,1]$  \\
      \midrule
      KAN Layers & & $[2,5,1]$  \\
      KAN Spline Order & & $3$ \\
      KAN Coarse Grid Size \;\; & & $4$ \\
      \midrule
      Optimizer     & & Adam \\
      Learning Rate  & $\eta$ & $10^{-2}$ \\
      Initial Learning Rate & $\eta_0$ &  $10^{-4}$ \\
      Ramping Steps & $S_{LR}$ & $10$ \\
      RBA Damping & $\mu$ & $10^{-4}$ \\
      \bottomrule
    \end{tabular}
    \addtolength{\tabcolsep}{0.4em}
    }

The 1d Burger's example uses random number generator seeds that start at 1234 and increase by 10 for each new run.

    {
\vspace{12pt}  
    \addtolength{\tabcolsep}{-0.4em}
    \begin{tabular}{lll}
      \multicolumn{3}{c}{\bf Section 5.2.2: 1d Burger's Hyper-Parameters} \\
      \toprule
      {\bf Parameter\;\;}      & {\bf Symbol} \;\; & {\bf Value} \\
      \midrule
        Loss Scaling & $\Gamma_V$ & $10^{0}$ \\
                     & $\Gamma_B$ & $10^{0}$ \\
      Volume Points & $|V|$ & 4096 \\
      Initial Condition Points & $|B|$ & 64 \\
      \midrule
      Optimizer     & & Adam \\
      Learning Rate  & $\eta$ & $10^{-3}$ \\
      Cycle & $S_{LR}$ & $100$ \\
      Exponential Decay Parameter & $\gamma$ & $0.9995$ \\
      \bottomrule
    \end{tabular}
    \addtolength{\tabcolsep}{0.4em}
    }

The Allen-Cahn example uses the random number generator seed 1234. It uses the Adam optimizer with a linear learning rate scheduler over the first 10 epochs to increase the learning rate from $10^{-4}$ to $10^{-2}$.

    {
\vspace{12pt}  
    \addtolength{\tabcolsep}{-0.4em}
    \begin{tabular}{lll}
      \multicolumn{3}{c}{\bf Section 5.2.3: Allen-Cahn Hyper-Parameters} \\
      \toprule
      {\bf Parameter\;\;}      & {\bf Symbol} \;\; & {\bf Value} \\
      \midrule
        Loss Scaling & $\Gamma_V$ & $10^{-1}$ \\
                   & $\Gamma_B$ & $10^{0}$ \\
      Volume Points & $|V|$ & 251001 \\
      Initial Condition Points & $|B|$ & 501 \\
      \midrule
      Optimizer     & & Adam \\
      Learning Rate  & $\eta$ & $10^{-2}$ \\
      \bottomrule
    \end{tabular}
    \addtolength{\tabcolsep}{0.4em}
    }

\section{Allen-Cahn ReLU training}

See \Cref{fig:pinns:allencahn:relu}.
\begin{figure}[H]
    \centering
    \includegraphics[width=0.95\textwidth]{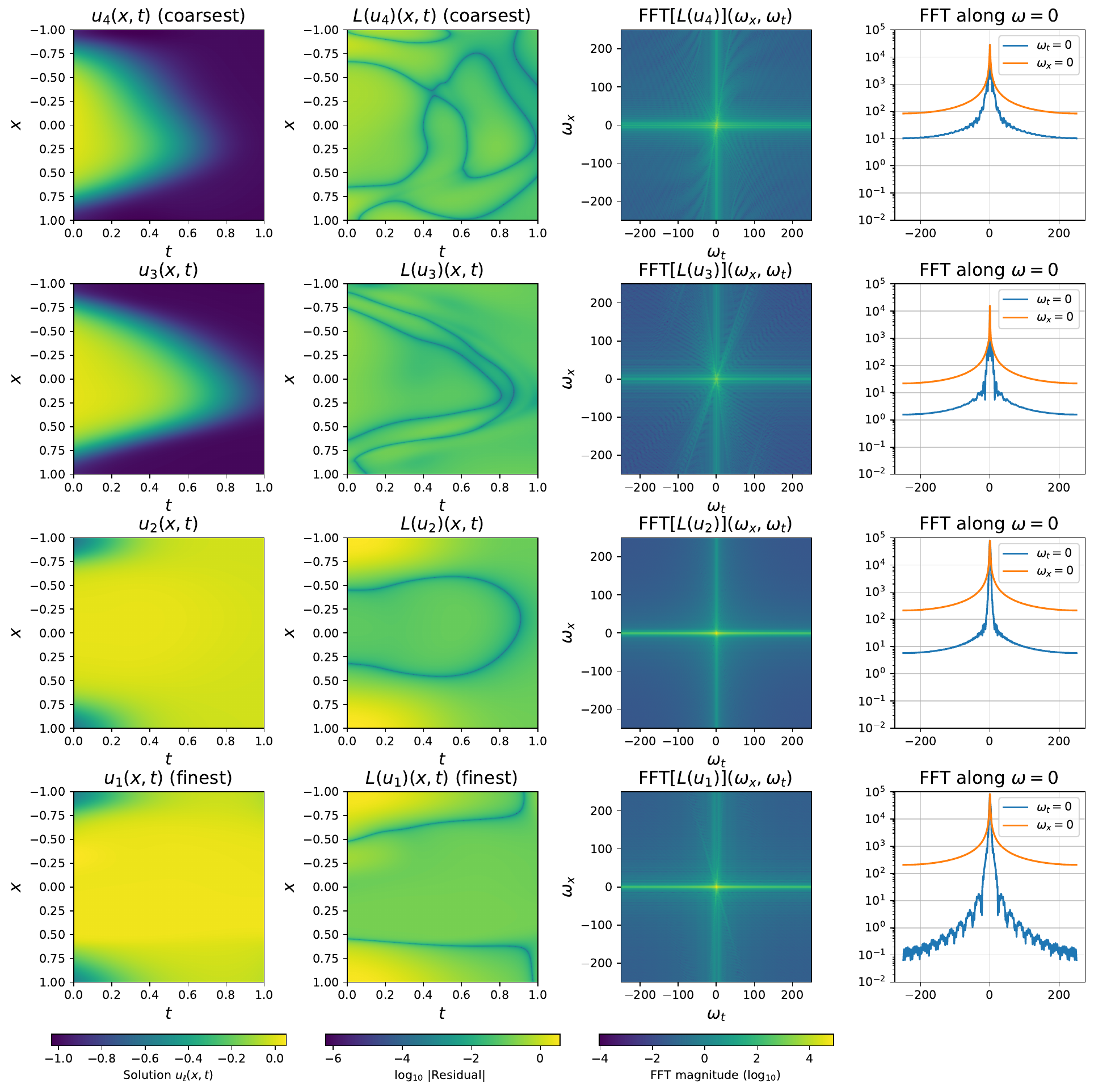}
    \caption{Plots of the trained ReLU-basis solution (leftmost) for the Allen-Cahn Equation, residual in log scale (left center), the Fourier transform of residual (right center), and cross-sections of the Fourier transform along the axes $\omega_t = 0$ and $\omega_x=0$ (rightmost). Each row displays the output after training at each level of refinement, from coarsest (top) to finest (bottom). The first three columns share the same colorbar shown at the bottom of the column.  
    \label{fig:pinns:allencahn:relu}}
\end{figure}

\end{document}